\documentclass{amsart}

\usepackage[
    pagebackref,
    colorlinks=true,
    citecolor=red,
]{hyperref}
\usepackage{geometry}
\usepackage{amsmath}
\usepackage{amssymb}
\usepackage{mathtools}
\usepackage{amsthm}

\usepackage[capitalize,noabbrev]{cleveref}
\usepackage[foot]{amsaddr}
\usepackage{enumitem}
\usepackage{booktabs}

\usepackage{natbib}

\newcommand{\RR}{\mathbb{R}} 
\newcommand{\NN}{\mathbb{N}} 
\newcommand{\EE}{\mathbb{E}} 
\newcommand{\PP}{\mathbb{P}} 
\newcommand{\calL}{\mathcal{L}} 

\newcommand{\calO}{\mathcal{O}} 
\DeclareMathOperator{\dg}{dg} 
\DeclareMathOperator{\Id}{I} 
\DeclareMathOperator{\conv}{conv} 
\newcommand{\bydef}{\stackrel{\mbox{\tiny\textnormal{\raisebox{0ex}[0ex][0ex]{def}}}}{=}}

\theoremstyle{plain}
\newtheorem{theorem}{Theorem}[section]
\newtheorem*{theorem*}{Theorem}
\newtheorem{proposition}[theorem]{Proposition}
\newtheorem{lemma}[theorem]{Lemma}
\newtheorem{corollary}[theorem]{Corollary}
\theoremstyle{definition}

\newtheorem{assumption}[theorem]{Assumption}
\newtheorem*{assumption*}{Assumption}
\theoremstyle{remark}
\newtheorem{remark}[theorem]{Remark}

\usepackage{todonotes}
\usepackage{subcaption}
\usepackage{graphicx}

\title[Improved Overparametrization Bounds for Global SGD Convergence]{Improved Overparametrization Bounds for Global Convergence of SGD for Shallow Neural Networks}

\author{Bart\l{}omiej Polaczyk}
\address{Institute of Informatics, University of Warsaw, Banacha 2, Warsaw, Poland.}
\author{Jacek Cyranka}
\email{b.polaczyk@mimuw.edu.pl (BP)}
\email{j.cyranka@mimuw.edu.pl (JC)}

\begin{document}

\begin{abstract}
	We study the overparametrization bounds required for the global convergence of stochastic gradient descent algorithm for a class of one hidden layer feed-forward neural networks equipped with ReLU activation function. 
	We improve the existing state-of-the-art results in terms of the required hidden layer width. 
	We introduce a new proof technique combining nonlinear analysis with properties of random initializations of the network.
\end{abstract}
\maketitle 

\tableofcontents

\section{Introduction}
\label{sec:introduction}
The study of convergence properties of mini-batch stochastic gradient descent (SGD) iterations applied to feed-forward neural nets (NN) is at the core of modern machine learning research. 
SGD with its variants like ADAM is the most common optimization scheme applied for supervised training of~NN. 
In principle however, the loss landscape encountered when training NN is highly nonconvex, especially for deep nonlinear NN as revealed, e.g., by visualizations performed by~\cite{visualize}, and construction proofs of spurious local minima by \cite{explocmin,brutzkus2018sgd}. 
The nonconvexity may have severe consequences for practical NN training routines, as SGD may potentially get stuck at a spurious local minimum or a saddle point and cease to converge further down the loss valley. 
Yet, practice suggests that with enough overparametrization, SGD iterations achieve global minima most of the times.
This phenomenon is not fully understood yet and is the main theme of this paper.

Contemporary research on NN convergence theory was initiated with the study of linear networks. The loss landscape in this setting was fully characterized by~\cite{kawaguchi}, solving the problem stated by~\cite{choromanska}. 
The research revealed the feasibility of global SGD convergence for deep NN despite the loss landscape nonconvexity.

Even though it seems difficult to fully characterize the loss landscape in the nonlinear setting, proving the global SGD convergence is still feasible. 
Recent research suggests that SGD converges globally with high probability for random initialization of weights, under the assumption of sufficiently large overparametrization expressed in terms of NN layers' widths. 
The first result of this kind required an unrealistic level of overparametrization of polynomial order in the number of samples, cf.~\cite{allenzhu}. 
The following series of related results (see Table~\ref{tab:related_works}) further reduced the required level of overparametrization using various techniques and assumptions on training data. 
Especially in the case of Deep NN equipped with analytic activation functions, an overparametrization of the linear order with respect to the number of training examples is sufficient. 
However, such tight overparametrization results do not apply in the case of a non-differentiable ReLU activation function (see Table~\ref{tab:related_works}).
Existing theoretical bounds still require a significantly larger number of parameters than used in practice. The question about an exact boundary marking the minimal number of parameters required for the global convergence is still open even for shallow (one hidden layer) ReLU NN, see, e.g.,~\cite{oymak}.

\subsection{Main Contribution.}
We establish a new theoretical order of overparametrization required for SGD convergence towards a global minimizer for one hidden layer NN with ReLU activations, improving known state-of-the-art bounds. 
We introduce a new proof technique based on nonlinear analysis.  
First, we show the global convergence of continuous solutions of the differential inclusion (DI) being a nonsmooth analog of the gradient flow for the MSE loss. 
Second, using the existing nonsmooth analysis results, we establish closeness of continuous trajectories to SGD sequences until convergence for a sufficiently small learning rate. 

The concept of studying the dynamics of continuous solutions pursued in this work already appeared earlier~\cite{arora2018a,du_provably}. 
However, the authors treated the convergence of SGD sequences independently from the analysis of continuous solutions, which served motivational purpose only.
We develop a rigorous method for for establishing the convergence of SGD sequences via the convergence of continuous solutions, which works for general nonsmooth approximators including deep NN and general loss functions.

\subsection{Informal statements.}
We derive the global convergence results under the following assumptions and notation (made precise later on). 
Let $N$ be the sample size.
The input data comes from the i.i.d. sub-Gaussian distribution on the sphere in $\RR^{d_0}$, where $d_0\in[N^{\delta_0}, N]$ for some $\delta_0\in(0,1)$. 
The initial weight vector $\theta_0$ is obtained via LeCun scheme (variance scales with width). 
$\mathcal{L}(\theta)$ is the MSE loss for some output matrix, weight vector $\theta$ and NN equipped with ReLU activation function.
The subdifferential in the sense of Clarke is denoted by $\partial$ and $\tilde{\Omega}$ is the $\Omega$ notation hiding the logarithmic terms.
All presented results hold with high probability (WHP), meaning that the probability of the event converges to one as the number of samples $N$ diverges to infinity, a convention widely adopted in the literature.

Our first main result provides a condition for the global convergence of the continuous solutions of the nonsmooth analog of gradient flow for $\mathcal{L}$.
\begin{theorem}[Informal Corollary~\ref{C:g_flow_convergence}]	\label{T:1_informal}
	Let the width of the shallow NN satisfy 
	$
	d_1 = \tilde{\Omega}(N^{1.25}).
	$
	Then, any solution $\theta\colon\RR_+\to\RR$ to the DI Cauchy problem  $\theta(0) = \theta_0$, $\dot{\theta}(t) 
	\in - \partial \calL(\theta(t))$
	satisfies 
	$
	\calL(\theta(t))\le 
	\calL(\theta(0)) \exp
	(
	-ctd_1
	)
	$
	for all $t\ge 0$ and some constant $c>0$ WHP.
\end{theorem}
The second main result establishes the global convergence for the mini-batch SGD iterates WHP.
\begin{theorem}[Informal Theorem~\ref{T:GD_main}]\label{T:2_informal}
	Let the width of the shallow NN satisfy 
	$d_1 = \tilde{\Omega}(N^{1.25})$.
	Then, for any error $\varepsilon > 0$ and any mini-batch size, the mini-batch SGD sequences with step size small enough achieve the loss value below $\varepsilon$ at a linear convergence rate WHP.
\end{theorem}
We obtain Theorem~\ref{T:2_informal} via the following result. 
It is stated for general approximators  (including deep ReLU NN) and general losses (including hinge loss, cross-entropy etc.). We believe it is of independent interest. 
We drop the assumption on the MSE loss and particular NN, and use the notion of an arbitrary loss $\tilde{\calL}$.

\begin{theorem}(Informal Theorem~\ref{P:abstract_convergence})\label{T:3_informal}
	Let the loss function $\tilde{\calL}$ be arbitrary satisfying some mild technical conditions.
	Additionally, assume there exists a nonempty compact set $Q$, s.t. any solution $\theta$ to the DI $\dot{\theta}(t)\in-\partial\tilde{\calL}(\theta(t))$ if initialized in $Q$, remains in some compact set $G$ and converges to zero as $\tilde{\calL}(\theta(t))\le \tilde{\calL}(\theta(0))e^{-\gamma t}$.
	Then, for any $\varepsilon>0$, the SGD sequences initialized in $Q$ and with step size small enough achieve the loss value below $\varepsilon$ at a linear convergence rate WHP.
\end{theorem}

Let us comment briefly on some key aspects of our results.

\subsection*{Overparametrization Bound Improvement.}
Theorem~\ref{T:2_informal} improves state-of-the-art overparameterization bounds for global SGD convergence for shallow ReLU NN -- in Table~\ref{tab:related_works} we compare it to the selected works that we find most related.
For instance, \cite{nguyen} require $d_1=\Omega(N^2)$. 
Similarly, \cite{oymak} require $d_1=\Omega(N^4/d_0^3)$ (which is better for $d_0$ in a small neighborhood of $N$), where they train the first weight matrix only and the second weights matrix remains fixed, cf.~Remarks~\ref{R:scaling} and~\ref{R:improvement} for a detailed discussion and Section~\ref{sec:num_exps} for numerical experiments comparing both setups. We also note that we have more general data assumptions than \cite{oymak}.

On the other hand, results from~\cite{kawaguchi_huang} and~\cite{belkin} require only linear overparametrization and work for more general data. 
However, they do not apply to ReLU as they rely heavily on the smoothness of the activation function.
In particular, analysis of non-smooth activation
functions seems to be a much more challenging task, see
e.g., a result showing the existence of spurious local minima
in the ReLU setting~\cite{safranshamir}.

\subsection*{Discrete vs Continuous Convergence.}
The idea of establishing a link between continuous solutions to the gradient flow and their discrete GD analogs for deep linear networks was introduced recently by \cite{cohen}. 
Their method require the Hessian to exist and to be bounded along the continuous trajectories. 
Such approach does not work when a nonsmooth activation function, e.g. ReLU, is employed -- we provide additional evidence supporting this claim in Section~\ref{sec:num_exps}.
Our approach of passing from continuous solutions to SGD sequences is more general because it works in the differential inclusions setting, which treats nondifferentiable objectives (in contrast to~\cite{cohen}).

\subsection*{SGD step size.}
One should keep in mind that Theorem~\ref{T:2_informal} is qualitative -- it does not provide a constructive condition for the step size to guarantee convergence.
However, existing quantitative results for ReLU NNs give to the best of our knowledge no better bound than $\calO(1/N^2)$, which is still far from the learning rates used in ML practice.
\begin{table*}[th!]
	\scriptsize
	\centering
	\caption{A Perspective on related work.  Reported results using notation $\widetilde{\Omega}$ hides logarithmic terms, $N$ is the number of train samples, $d_0$ is the input dimension, $L$ is the number of layers of deep NN}
	\begin{tabular}{lccccc}
	\toprule
	\textbf{Work} & \textbf{Algorithm} & \textbf{ReLU} & \textbf{Deep} & \textbf{Data} & \textbf{Scaling}\\
	\midrule
	\cite{dulee} & GD & no & yes & $\begin{array}{c}\text{non degenerate} \\ \text{normalized}\end{array}$ & $\tilde{\Omega}(2^{O(L)}N^4)$\\[2pt]
	\begin{tabular}{@{}l@{}}
	\cite{kawaguchi_huang}
	\end{tabular} 
		& GD & no & yes & normalized & $\begin{array}{c}\tilde{\Omega}(Nd_0)\,\text{(shallow)}\\\tilde{\Omega}(N+d_0L^2)\,\text{(deep)}\end{array}$\\[7pt]
	\cite{belkin} & SGD  & no & yes & $\begin{array}{c}\text{non degenerate} \\ \text{normalized}\end{array}$ & $\tilde{\Omega}(N)$\\[2pt]
	\midrule
	\cite{allenzhu} & SGD & yes & no & separable & $\tilde{\Omega}(N^{24}L^{12})$\\[2pt]
	\cite{arora} & GD & yes & yes & unif. on sphere & $\tilde{\Omega}(N^7)$\\[2pt]
	\cite{zougu} & SGD & yes & yes & separable & $\tilde{\Omega}(N^8L^{12})$\\[2pt]
	\begin{tabular}{@{}l@{}}
	\cite{oymak}
	\end{tabular} 
	& $\begin{array}{c}\text{SGD}\\\text{(on layer 1)}\end{array}$ & yes & no & unif. on sphere & $\tilde{\Omega}(N^4/d_0^3)$\\[7pt]
	\cite{nguyen} & GD & yes & yes & subgaussian & 
	$\begin{array}{c}\tilde{\Omega}(N^2)\text{ (shallow)}\\\tilde{\Omega}(N^3)\text{ (deep)}\end{array}$\\[7pt]
	\midrule
	\textbf{Ours} & SGD & yes & no & subgaussian & $\tilde{\Omega}(N^{1.25})$\\
	\bottomrule
	\end{tabular}
	\label{tab:related_works}
\end{table*}

\subsection{Other Related Work.}
We summarize the current literature concerning the question of SGD global convergence for NN equipped with the MSE loss in Table~\ref{tab:related_works}. We split the results into two groups, first the ones working for smooth activations and second, the results for ReLU activation function, also considered in this work.
Similar and, in some cases, tighter overparametrization results have been established for training deep NN equipped with cross-entropy loss \cite{li_liang,Ji2020Polylogarithmic,chen2021how}.
All existing results are derived under the assumption that there is a significant overparametrization of the NN under study (at least one wide hidden layer). Earlier work focused on the non-existence of spurious local minima without consideration of SGD dynamics \cite{pmlr-v54-xie17a}. The extreme case of overparametrization, i.e., infinite layer width, has also been analyzed in \cite{opttransport_infinite,ntk_infinite,pdeconv}.

One can also find negative results in the literature, demonstrating, e.g., the existence of spurious local minima in underparameterized regimes,~\cite{loc_min}, or convergence towards spurious local minima,~\cite{brutzkus2018sgd}.
As for other fundamental properties, nonlinear NN are universal approximators~\cite{cybenko,shaham}. 
NN memorization property has also been extensively studied -- in the case of shallow NN, known overparametrization bounds for perfect memorization of the data are near-optimal \cite{zhang2016understanding,moritz_hardt,Nguyen_express,baldi,Yun_Sra,bubeck}.

\subsection{Organization of this paper}
In Section~\ref{sec:prelim} we introduce the notation and recall some facts regarding differential inclusions.
In Section~\ref{sec:di_dynamics} we study the properties of the DI solutions for MSE loss.
In Section~\ref{sec:di_globconv} we prove the global convergence result for DI solutions under random initialization.
In Section~\ref{sec:gradient} we extend the result of Section~\ref{sec:di_globconv} to SGD iterates.
In Section~\ref{sec:num_exps} we present some numerical experiments related to our results.
We summarize our findings in Section~\ref{sec:conclusion}.

\section{Preliminaries}
\label{sec:prelim}
Let $X\in \RR^{N\times d_0}$ be a matrix of the training inputs (arranged rowwise) and $Y\in\RR^{N\times d_2}$ be a matrix of training labels, where $N\in \NN_+ \bydef 1,2\ldots$ is the sample size and $d_0,d_2\in\NN_+$ are the dimensions of the input and output respectively. Consider the following one hidden-layer feed-forward NN
\begin{equation*}
\hat{Y} \bydef \phi(XW)V,
\end{equation*}
where for some $d_1\in\NN_+$, $W\in \RR^{d_0\times d_1}$ and $V\in \RR^{d_1\times d_2}$ are the weight matrices and $\phi\colon \RR\to\RR$ is the ReLU activation function applied element-wise.
We often assume that $X,Y$ are fixed and known from context, whence they are not explicitly mentioned as parameters, e.g., in the loss function formula.
We denote the hidden layer by $H$, i.e., $H\bydef \phi(XW)\in \RR^{N\times d_1}$. 
We write $D \bydef d_0d_1 + d_1d_2$ and denote parameter vector by $\theta \in \RR^D$, i.e., $\theta$ is obtained by stacking vectorized matrices $W,V$. 
We identify matrices with their vectorized forms and write simply $\theta = (W,V)$.
 
The standard dot product and Euclidean distance on $\RR^d$ for $d\in\NN_+$ are denoted by $\langle\cdot,\cdot\rangle$ and $\|\cdot\|$.
For $x\in\RR^{d}$ and $r>0$, $B(x,r)\bydef \{\, y\in\RR_{d}\colon \|y-x\|\le r \,\}$ is the closed ball with radius $r$ centered at $x$.
For a matrix $A\in \RR^{n_r\times n_c}$, $A_{i:}$ denotes the $i$-th row vector of $A$ for $i\in [n_r]$, and $A_{:i}$ denotes the $i$-th column vector of $A$ for $i\in [n_c]$, where $[k]\bydef \{1,\ldots,k\}$ for $k\in\NN_+$.
Finally, we denote the minimal eigen- and singular values of $A$ by $\lambda_{min}(A)$ and $\sigma_{min}(A)$, i.e., $\sigma_{min}(A) = \sqrt{\lambda_{min}(A^TA)}$, while the operator and Frobenius norms of $A$ are denoted by $\|A\|_{op}$ and $\|A\|_{F}$.

Our aim is to optimize the MSE loss function $\calL\colon\RR^{D}\to\RR_{+}$, defined via $
\calL(\theta) \bydef \frac{1}{2}\| Y - \hat{Y} \|_F^2$.
The widely applied ReLU activation function is non-differentiable at $x=0$ but the generalized derivative in the sense of Clarke, cf.~\cite{clarke1983oan}, exists and is equal to the interval $[0,1]$. 
We denote the Clarke subdifferential by $\partial$ and refer the reader to~\cite{rockafellar2009variational} for a detailed treatment of generalized gradients.

Recall that a curve\footnote{We use the same symbols to denote points and curves.} $x\colon \RR_+\to \RR^d$ is absolutely continuous if there exists a map $v\colon \RR_+ \to \RR^d$ that is integrable on compact intervals and s.t. $x(t) - x(0) = \int_0^t v(s)\,ds$ for all $t\ge 0$.
To lighten the notation we sometimes write $\frac{d}{dt}x(t) = \dot{x}(t)$ and call any absolutely continuous curve an arc.
We are interested in finding arcs $x$ that are solutions to the following differential inclusion Cauchy problem
\begin{equation}\label{eq:DI_generic}
x(0) = x_0,\quad 
\dot{x}(t) \in -\partial f(x(t))
\;
\text{ for a.e. }
t\ge 0,
\end{equation}
where $x_0 \in \RR^d$ and $f\colon \RR^d\to\RR$ are given.
The following property plays a crucial role in analyzing such problems -- we say that $f$ satisfies the \textit{chain rule} if for any arc $x\colon\RR_+\to\RR^d$,
\begin{equation}\label{eq:chain_rule}
\big\{\,
\langle v,\, \dot{x}(t) \rangle
\colon
v\in \partial f( x(t) )
\,\big\}
=
\big\{\, 
\frac{d}{dt}(f\circ x)(t)
\,\big\}
\quad 
\text{for a.e. } t\ge 0.
\end{equation}

Consider the dynamics given by the following DI obtained from~\eqref{eq:DI_generic} by taking $f=\calL$,
\begin{equation}\label{eq:dynamics}
\theta(0) = \theta_0,\quad
\dot{\theta}(t) 
\in - \partial \calL(\theta(t))
\;
\text{ for a.e. }
t\ge 0,
\end{equation}
where $\theta_0\in\RR^{D}$ is some initial value.
Note that a-priori it is unknown if there exists a solution to~\eqref{eq:dynamics} defined on the whole interval $[0,\infty)$.
Recall the notation $H=\phi(XW)$.
The following standard result is due to the fact that $\calL$ satisfies the chain rule~\eqref{eq:chain_rule}, cf.~\cite{davis2020}, combined with usual arguments regarding DIs, subdifferential of $\calL$ and Gr\"{o}nwall's lemma. 
Since we were unable to find such statement that rigorously treats its existential component connected to the theory of DIs, we provide its detailed proof in Appendix~\ref{A:pf_existence}.
\begin{proposition}\label{P:DI_existence}
	For any initial $\theta_0\in \RR^{D}$, there exists $T>0$ and a solution $\theta\colon [0,T)\to\RR^D$ to the DI~\eqref{eq:dynamics}. 
	Moreover, for any bounded domain $G\ni \theta_0$, each solution $\theta$ to~\eqref{eq:dynamics} can be extended to infinity or up until it hits the boundary of $G$.
	Finally, for any such $\theta$, denoting $\alpha_0(s) \bydef \sigma_{min}(H^T(\theta(s)))$, one gets
	$$
		\calL(\theta(t)) \le 
		\calL(\theta(0))
		\exp
		\big(
		-2\int_0^t \alpha_0^2(s)\,ds
		\big)\quad\text{for every $t\in[0,T)$.}
	$$ 	
\end{proposition}

\section{Dynamics of the Differential Inclusion}
\label{sec:dynamicsofdi}
\label{sec:di_dynamics}
In this section, we show that the integral of the loss (square root) along the parameter $\theta$ trajectories determined by the DI~\eqref{eq:dynamics} satisfies a simple one-dimensional differential inequality.
From that we infer boundedness properties of the loss along trajectories.
The constants appearing in the differential inequality depend on the initialization properties only which allows us to provide WHP estimates in Section~\ref{sec:di_globconv}. 

Recall the notation $H = \phi(XW)$ and $\alpha_0(s)=\sigma_{min}(H^T(\theta(s)))$.
By Weyl's inequality, cf., e.g.,~\cite[Theorem~4]{dax2013eigenvalues}, and Lemma~\ref{L:mtx_product_norm_inequality},
\begin{align}\label{eq:Weyl}
\begin{split}
\vert \alpha_0(t) - \alpha_0(0) \vert 
\le 
\| H(t) - H(0) \|_{F}
\le 
\| X(W(t) - W(0)) \|_{F}
\le 
\| X \|_{op} \| W(t) - W(0) \|_{F}.
\end{split}
\end{align}
Therefore, to use Proposition~\ref{P:DI_existence}, in lemma below we bound the quantity $\|X\|_{op}\| W(t) - W(0) \|_{F}$.
We defer its proof, which is based on a careful application of Gr\"{o}nwall's lemma, to Appendix~\ref{A:pf_W_increments}.
\begin{lemma}\label{L:W_increments_bound}
	Any solution $\theta\colon [0,T)\to\RR$, $\theta = (W,V)$, to the DI~\eqref{eq:dynamics} satisfies
	\begin{equation}\label{eq:theta_increment_bound_result}
	\| \theta(t) - \theta(0) \|
	\le 
	\sqrt{2}\|X\|_{op} 
	\big(\|W(0)\|_{F} + \|V(0)\|_{F}\big)	
	\bar{\calL}(t)
	\exp\big(
	\sqrt{2}\|X\|_{op} \bar{\calL}(t)
	\big)
	\end{equation}
	for every $t\in[0,T)$, where $\bar{\calL}(t) = \int_0^t\sqrt{\calL(\theta(s))}\,ds$.
	Moreover
	\begin{equation}\label{eq:W_increment_bound_result}
	\|X\|_{op} \| W(t) - W(0) \|_{F}
	\le 
	\frac{1}{2}
	\Big(
	c_1 \bar{\calL}(t) 
	+
	c_2 \big( \bar{\calL}(t) \big)^2
	\Big)
	\exp\Big(
	c\big(\bar{\calL}(t) \big)^{2}
	\Big)
	\end{equation} 
	for every $t\in[0,T)$, where
	\begin{equation}\label{eq:c1_c2_c_def}
	c_1 \bydef 2\sqrt{2}\|X\|_{op}^2\|V(0)\|_F,\quad 
	c_2 \bydef 2\|X\|_{op}^{3}\| W(0) \|_{F},\quad 
	c \bydef \|X\|_{op}^2.
	\end{equation}
\end{lemma}

Using Lemma ~\ref{L:W_increments_bound} and Proposition~\ref{P:DI_existence} we infer in the proposition below that loss trajectories along solutions to the DI~\eqref{eq:dynamics} obey some specific differential inequality.
This observation is crucial for obtaining the main results of this paper, i.e., Corollary~\ref{C:g_flow_convergence} and Theorem~\ref{T:GD_main}.
\begin{proposition}\label{P:calL_solves_DI}
	Let $c,c_1,c_2$ be as in Lemma~\ref{L:W_increments_bound},~\eqref{eq:c1_c2_c_def}.
	Set 
	\begin{equation}\label{eq:a_alpha_def}
		a \bydef \sqrt{\calL(\theta(0))}, 
		\quad 
		\alpha \bydef \sigma_{min}(H^T(\theta(0))).
	\end{equation}
	If for some $T>0$, $\theta\colon[0,T)\to\RR^{D}$ solves the DI~\eqref{eq:dynamics}, then $\bar{\calL}(t) \bydef \int_0^t\sqrt{\calL(\theta(s)}\,ds$ is a solution $y\colon[0,T)\to\RR$ to the problem
	\begin{align}\label{eq:ODE_problem}
	y(0) = 0;
	\quad 
	y'(t) \le 
	a \exp
	\big(
	\alpha t 
	(
	c_1y(t)
	+
	c_2y^2(t)
	)
	e^{cy^2(t)}
	-\alpha^2 t
	\big)
	\;
	\text{ for all }
	\;
	t\in [0,T).
	\end{align}
\end{proposition}
\begin{proof} 
Using Proposition~\ref{P:DI_existence}, 
the inequality $(u-v)^2\ge u^2 - 2u\vert v\vert$ for $u\ge 0$, $v\in\RR$, 
and the estimate from~\eqref{eq:Weyl},
we get for all $t\in [0,T)$,
\begin{align*}
\begin{split}
	\sqrt{\calL(\theta(t)}
	&\le 
	\sqrt{\calL(\theta(0))}
	\cdot 
	\exp\Big(
	-\int_0^t\alpha_0^2(s)\,ds
	\Big)
	\\&\le 
	\sqrt{\calL(\theta(0))}
	\cdot 
	\exp\Big(
	-t\alpha_0^2(0) +
	2\alpha_0(0)
	\int_0^t\vert \alpha_0(s) - \alpha_0(0) \vert \,ds
	\Big)
	\\&\le 
	\sqrt{\calL(\theta(0))}
	\cdot 
	\exp\Big(
	-t\alpha_0^2(0)
	+2\alpha_0(0)
	\int_0^t \|X\|_{op}\|W(s)-W(0)\|_F\,ds
	\Big).
\end{split}
\end{align*}
Using the bound from~\eqref{eq:W_increment_bound_result} due to Lemma~\ref{L:W_increments_bound} and noting that $\bar{\calL}'(t)=\sqrt{\calL(\theta(t))}$, we arrive at 
\begin{align*}
\begin{split}
	\bar{\calL}'(t)
	&\le 
	a
	\cdot 
	\exp\Big(
	-t\alpha^2
	+2\alpha
	\int_0^t \|X\|_{op}\|W(s)-W(0)\|_F\,ds
	\Big)
	\\&\le 
	a\cdot
	\exp\Big(
	-t\alpha^2
	+\alpha
	\int_0^t
	\big(
	c_1 \bar{\calL}(s) 
	+
	c_2 \big( \bar{\calL}(s) \big)^2
	\big)
	\exp\big(
	c\big(\bar{\calL}(s) \big)^{2}
	\big)
	\,ds
	\Big)
\end{split}
\end{align*}
and the conclusion follows by estimating $\bar{\calL}(s)\le \bar{\calL}(t)$ for all $s\in[0,t]$.
\end{proof}
Perhaps surprisingly, due to the double exponential dependence on $y^2(t)$, a simple condition involving $a,c,c_1,c_2,\alpha$ determines that solutions to~\eqref{eq:ODE_problem} remain bounded by $2a/\alpha^2$ for all times, as demonstrated in Lemma~\ref{L:ODE_solutions} below. 
This property is illustrated in Figure~\ref{fig:ode}.
\begin{lemma}\label{L:ODE_solutions}
	Let $a,\alpha,c,c_1,c_2$ be some arbitrary parameters of~\eqref{eq:ODE_problem}. 
	If $\alpha>0$ and
	\begin{equation}\label{eq:ODI_condition}
	4
	\Big(
	\frac{ac_1}{\alpha^3}
	+
	\frac{2a^2c_2}{\alpha^5}
	\Big)
	\exp\big(
	4ca^2/\alpha^4
	\big)
	<
	1,
	\end{equation}
	then for any $T>0$, any solution $y\colon[0,T)\to\RR$ to the problem from~\eqref{eq:ODE_problem} is bounded from above by $2a/\alpha^2$ and its derivative at any time $t\in [0,T)$ is bounded by $ae^{-\alpha^2t/2}$.
\end{lemma}
\begin{proof}
	Let $y\colon [0,T)\to\RR$ be any solution to~\eqref{eq:ODE_problem}.
	Set
	\[
	t_0 = \inf\big\{\,
	t\in[0,T)\colon \;
	\alpha\big( 
	c_1 y(t)
	+c_2y^2(t)
	\big) e^{cy^2(t)} = \alpha^2/2
	\,\big\}.
	\]
	By assumption $y(0)=0$ and $\alpha>0$, whence by continuity of $y$, $t_0>0$.
	Moreover, for a.e. $t< t_0$, $y'(t) \le  ae^{-\alpha^2 t/2}$, whence $y(t)\le 2a/\alpha^2 \cdot (1 - e^{-\alpha^2 t/2}) < 2a/\alpha^2$ for all $t<t_0$.
	If $t_0<T$, then by continuity $y(t_0)\le 2a/\alpha^2$ as well, whence
	\[
	\alpha^2/2 
	= 
	\alpha\big( 
	c_1 y(t_0)
	+c_2y^2(t_0)
	\big) e^{cy^2(t_0)}
	\le 
	\alpha
	\Big(
	\frac{2ac_1}{\alpha^2}+
	\frac{4a^2c_2}{\alpha^4}
	\Big)
	\exp\big(
	4ca^2/\alpha^4
	\big)
	\] 
	but this yields a contradiction with~\eqref{eq:ODI_condition}.
	Therefore $t_0=T$ as desired.
\end{proof}

Using Proposition~\ref{P:calL_solves_DI} in conjunction with Lemma~\ref{L:ODE_solutions}, we obtain in the theorem below the announced global convergence guarantee for continuous parameter trajectories.
\begin{figure}[h!]
\centering
\includegraphics[width=0.75\textwidth,trim=.5cm .5cm .5cm .5cm,clip]{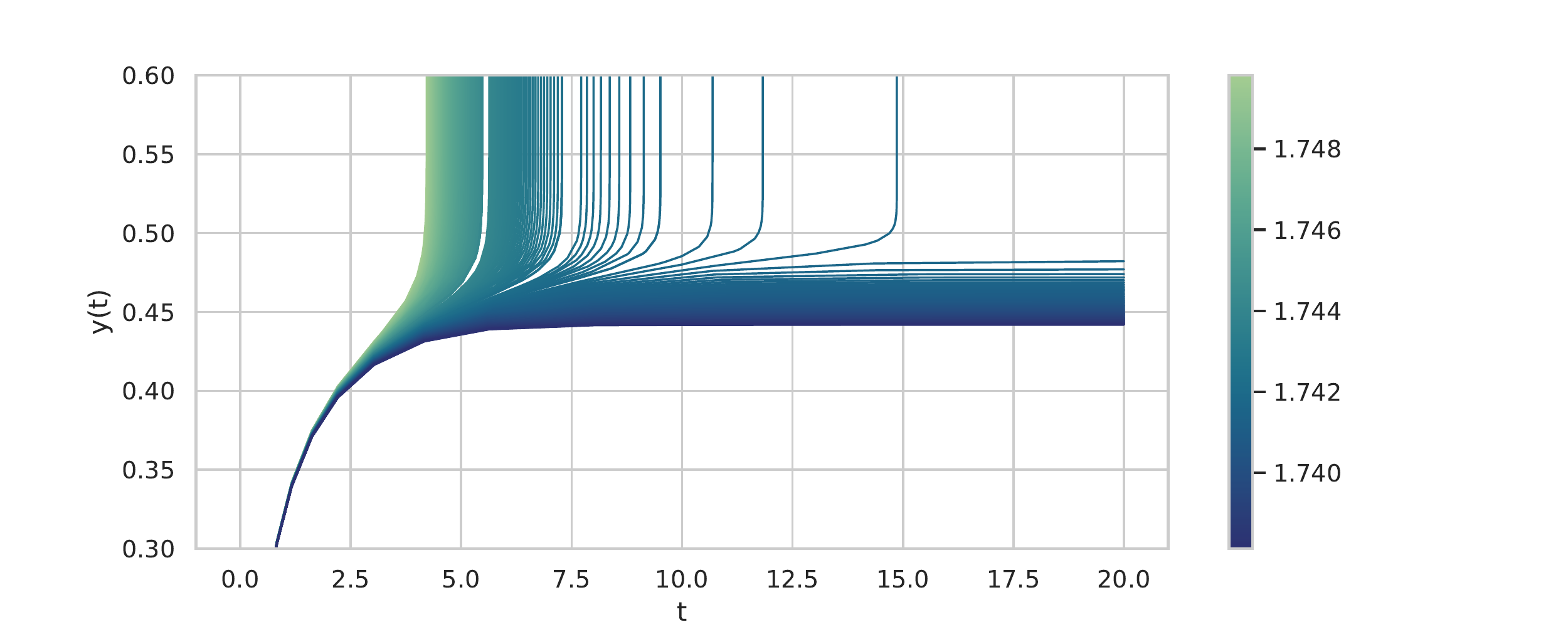}
\caption{
Numerical illustration of the solutions to~\eqref{eq:ODE_problem} for fixed $a, c_1,c_2,c=1$ and $\alpha\in[2.042, 2.045]$. The color scale encodes the values of $4({ac_1}/{\alpha^3}+{2a^2c_2}/{\alpha^5})\exp(4ca^2/\alpha^4)$, the quantity determining \eqref{eq:ODI_condition}. It is visible that the solution $y(t)$ either remains bounded or explodes rapidly depending on the condition involving the constants. Observe that empirically the upper bound is smaller than the derived theoretical bound.}
\label{fig:ode}
\end{figure}
\begin{theorem}\label{T:deterministic_convergence_guarantee}
	Let $a,\alpha,c,c_1,c_2$ be as in Proposition~\ref{P:calL_solves_DI}.
	Assume that $\alpha>0$ and that at initialization 
	\begin{equation}\label{eq:conditions_original_params}
	F(\theta(0), X, Y)
	\bydef 
	\big(
	\frac{ac_1}{\alpha^3}
	+
	\frac{a^2 c_2}{\alpha^5}
	\big)
	\exp\big(
	\frac{4ca^2}{\alpha^4}
	\big)
	< \frac{1}{8}.
	\end{equation}
	Then, any solution $\theta\colon[0,T)\to\RR^D$ to the DI~\eqref{eq:dynamics} can be extended to a solution on $\RR_+$ and any such extension satisfies for all $t\ge 0$,
	\begin{equation}\label{eq:deterministic_exponential_convergence}
	\calL(\theta(t))\le \calL(\theta(0))
	\exp(
	-t\alpha_0^2(0)
	)
	\end{equation}
	and for $u\bydef 4\|X\|_{op}\sqrt{\calL(\theta(0))}/\alpha_0(0)^2$,
	\begin{equation}\label{eq:theta_increments_initial_params_bound}
	\| \theta(t) - \theta(0) \|
	\le
	u\|\theta(0)\|
	e^{u}.
	\end{equation}
\end{theorem}
\begin{proof}
	If $\calL(\theta(0)) = 0$, then the result holds.
	If $\calL(\theta(0)) > 0$, then set 
	$
	U(s) \bydef \sqrt{2}\|X\|_{op}
	\big(\|W(0)\|_{F} + \|V(0)\|_{F}\big) s \cdot e^{\sqrt{2}\|X\|_{op} s}
	$
	and let $G \bydef B(\theta(0),2U(2a/\alpha^2))$.
	By Proposition~\ref{P:DI_existence}, there exists $T>0$ and a solution $\theta\colon [0,T)\to\RR^{D}$ to the DI~\eqref{eq:dynamics}, which can be extended up until it hits the boundary of $G$.
	Assume that $\theta$ is already such an extension.
	By Proposition~\ref{P:calL_solves_DI}, $\bar{\calL}(t)=\int_0^t\sqrt{\calL(\theta(s))}\,ds$ solves~\eqref{eq:ODE_problem}, whence Lemma~\ref{L:ODE_solutions} asserts that if $\alpha>0$ and~\eqref{eq:conditions_original_params} is satisfied, then $\bar{\calL}(t)$ is bounded from above by $2a/\alpha^2$ and ${\bar{\calL}'(t)}=\sqrt{\calL(\theta(t))}$ is bounded from above by $ae^{-t\alpha^2/2}$ for all $t\in[0,T)$.
	
	By Lemma~\ref{L:W_increments_bound},
	$
	\| \theta(t) - \theta(0) \|
	\le
	U(\bar{\calL}(t))
	\le 
	U(2a/\alpha^2)
	$
	for all $t\in [0,T)$, so $\theta$ never reaches the boundary of $G$, whence $T=\infty$ and~\eqref{eq:deterministic_exponential_convergence} follows. 
	Estimating $\|W(0)\|_F+\|V(0)\|_F \le \sqrt{2}\|\theta(0)\|$ gives~\eqref{eq:theta_increments_initial_params_bound}.
\end{proof}

\section{Convergence of the Differential Inclusion Trajectories}
\label{sec:di_globconv}
To verify that~\eqref{eq:conditions_original_params} holds WHP at initialization, we need to impose some additional assumptions on the data matrices $X$, $Y$, and on the initialization scheme of $\theta_0$.
In this section, all the complexity notations $\calO$, $\Omega$, $\Theta$, etc., are understood in terms of $N$ approaching infinity, e.g., for any space $\mathcal{X}$ and a function $f\colon\mathcal{X}\times \NN \to \RR$, we say that $f(x,N) = \calO(N)$ if $|f(x,N)|\le CN$ for some constant $C>0$ and all $x\in\mathcal{X}$.

Recall that a random variable $z\in\RR$ is sub-Gaussian if its Orlicz norm defined as
$\|z\|_{\psi_2} \bydef \inf\{ t > 0\colon \EE \exp(z^2/t^2) \le 2\}$
is finite.
A random vector $Z\in\RR^n$ is said to be sub-Gaussian if 
$\|Z\|_{\psi_2} \bydef \sup_{t\in\RR^n,\, \| t \|_2 =  1} \| \langle Z,t \rangle\|_{\psi_2}$ is finite.
For more refined treatment of the Orlicz norms and sub-Gaussian random variables, we refer the reader to~\cite{Vershynin_hdp}.

In the sequel, we impose the following assumption.

\begin{assumption}\label{A:init_assumptions}
	\hfill
	\begin{enumerate}[leftmargin=2em]
		\item $X_{i:}$'s are random i.i.d. sub-Gaussian vectors s.t. $\|X_{i:}\|_2=\sqrt{d_0}$ and $\|X_{i:}\|_{\psi_2}=\calO(1)$ for $i\in [N]$.
		\item $(W_0)_{ij}\sim \mathcal{N}(0,\beta^2_w)$ for $(i,j)\in [d_0]\times [d_1]$ and some $\beta_w>0$.
		\item $(V_0)_{ij}\sim \mathcal{N}(0,\beta^2_v)$ for $(i,j)\in [d_1]\times [d_2]$ and some $\beta_v>0$. 
		\item $W_0$ and $V_0$ are independent random vectors. 
		\item $\| Y_{i:} \|_2 = \calO(\beta_w\beta_v\sqrt{d_0d_1d_2})$ for $i\in[N]$.
	\end{enumerate}
\end{assumption}

\begin{remark}
	The choice of data scaling in Assumption~\ref{A:init_assumptions} is made merely to simplify the notation. 
	In particular, it asserts that under the LeCun initialization, 
	$\|Y_{i:}\|= \calO(\sqrt{d_2})$ for any $i\in[N]$ and that 
	$\|\hat{Y}\|_F$ is WHP of similar order as 
	$\| {Y}\|_F$ at initialization, cf. Lemma~\ref{L:init_loss}.
\end{remark}

The result below provides a lower bound on $\alpha_0(0)$. The proof is a slight modification of the argument from~\cite[Theorem~5.1]{nguyen2021tight} -- we present it in Appendix~\ref{A:pf_alpha_init}.
\begin{theorem}\label{T:alpha_0_init}
	Under Assumption~\ref{A:init_assumptions}, let $d_0 \in [N^{\delta_0}, N]$ for some $\delta_0\in(0,1)$.
	Let $\Psi\colon\NN\to[1,\infty)$ be s.t. 
	$
	d_1 \ge \max\big( 
	N, 
	C(\delta_0) 
	d_0^{-1}{N \Psi(N)\log^2(N)}
	\big)
	$
	for some  and $C({\delta_0})>0$ depending on $\delta_0$ only.
	Then, there exists a universal constant $c(\delta_0)$ depending on $\delta_0$ only, s.t.
	$
	\alpha_0(0)
	\ge 	
	\sqrt{c(\delta_0)d_0d_1}\beta_w
	$ 
	holds with probability at least
	$
	1-\exp( -\Psi(N) )
	-\calO(N^2)\exp( -\Omega(N^{\delta_0/2}) ).
	$
\end{theorem}

The following lemma follows from standard concentration inequalities -- we provide the proof for completeness in Appendix~\ref{A:pf_lem_concentration}.
\begin{lemma}\label{L:init_loss}
	If Assumption~\ref{A:init_assumptions} is satisfied, then 
	$$
		\|W_0\|_F = \Theta(\sqrt{d_0d_1}\beta_w)
		\quad\text{and}\quad
		\|V_0\|_F = \Theta(\sqrt{d_1d_2}\beta_v)
	$$
	with probability
	$
	1 -
	2\exp(-\Omega(d_0d_1)) -
	2\exp(-\Omega(d_1d_2)),
	$
	$$
	\|X\|_{op} = \calO(\sqrt{\max\{N,d_0\}})
	$$
	with probability $1-\exp(-\Omega(\max\{N,d_0\}))$,
	and
	$$
	\calL(\theta_0) = \calO(\|Y\|_F^2 + \beta_v^2 d_2 \|W_0\|_F^2\|X\|_{op}^2 \log(N))
	$$
	with probability. 
	$
	1-\exp(-\Omega(d_2 \log(N)))
	$.
\end{lemma}
Combining results from Sections~\ref{sec:di_dynamics} and \ref{sec:di_globconv} we obtain the following result demonstrating the global convergence of solutions to DI~\eqref{eq:dynamics} towards zero loss under initialization satisfying Assumption~\ref{A:init_assumptions}.
The full proof is provided in Appendix~\ref{A:pf_g_flow_convergence}.
\begin{corollary}\label{C:g_flow_convergence}
	Under Assumption~\ref{A:init_assumptions}, let $\beta_v^2=d_1^{-\rho}$ for some $\rho\ge 0$ and $d_0 \in [N^{\delta_0},N]$ for some $\delta_0\in(0,1)$. 
	Let moreover $c(\delta_0)$ and $C(\delta_0)$ be as in Theorem~\ref{T:alpha_0_init} and 
	\[
	d_1 \ge \max\big(N, 
	C(\delta_0) 
	\left[ 
	\frac{d_2N^{2.5}}{d_0\beta_w^2}
	\right]^{1/(1+\rho)}\log^2(N)
	\big).
	\]
	Then, any solution $\theta\colon[0,T)\to\RR$ to the DI~\eqref{eq:dynamics} can be extended to a solution on $\RR_+$ and any such extension satisfies 
	\begin{equation*}
	\calL(\theta(t))\le 
	\calL(\theta(0))\cdot \exp
	\big(
	-t\cdot c(\delta_0) d_0d_1\beta_w^2
	\big)
	\end{equation*}
	for all $t\ge 0$ with probability at least
	$
	1-
	\exp\big( 
	-
	\frac{d_0}{N}
	\cdot 
	\big[
	\frac{d_2N^{2.5}}{d_0\beta_w^2} 
	\big]^{1/{(\rho+1)}}
	\big)
	-\calO(N^2)\exp( 
	-\Omega(N^{\delta_0/2}) 
	)
	-\exp(
	-\Omega(d_2\log N)
	).
	$
\end{corollary}
\begin{proof}[Proof sketch]
	By Theorem~\ref{T:deterministic_convergence_guarantee} and Theorem~\ref{T:alpha_0_init}, it suffices to verify that $F(\theta(0), X, Y) = o(1)$.
	The last condition is verified WHP by means of Theorem~\ref{T:alpha_0_init} and Lemma~\ref{L:init_loss}.
\end{proof}

\section{Convergence of the Stochastic Gradient Descent Iterations}
\label{sec:gradient}
Let us consider a discrete version of the dynamics given by the DI Cauchy problem~\eqref{eq:dynamics}, i.e., the stochastic gradient descent.
We start with introducing some additional notation.

Let $(\Xi,\mathcal{F},\mu)$ be a probability space and consider a function $f\colon\RR^D\times \Xi \to\RR$, s.t. $f(\cdot, s)$ is locally Lipschitz for all $s\in\Xi$.
Let $\theta_0\in\RR^D$ be a random variable with absolutely continuous distribution function. 
For a fixed stepsize $\eta>0$, we say that a sequence of $\RR^D$-valued random variables $(\theta_k^{\eta})_{k\in\NN}$ is an $f$-SGD sequence if
\begin{equation}\label{eq:GD}
\theta_0^{\eta} = \theta_0;
\ 
\theta_{k+1}^{\eta} 
\in  
- \eta \cdot \partial f(\theta_k^{\eta}, \xi_{k+1})
\;
\text{for }
k\in\NN,
\end{equation}
where $\partial f(\theta, s)$ is the Clarke subdifferential at point $\theta$ applied to the function $\theta\mapsto f(\theta, s)$ and $(\xi_k)_{k\in\NN_+}$ is a sequence of i.i.d. $\Xi$-valued random variables distributed according to $\mu$, which are independent of $\theta_0$. 

For $b\in[N]$, let $[N]^{(b)}$ denote the family of subsets of $[N]$ containing exactly $b$ elements and $A_b \sim \operatorname{Unif}([N]^{(b)})$ be a random variable selecting each item from $[N]^{(b)}$ with the same probability.
We define the loss function $\calL^b\colon\RR^{D}\times [N]^{(b)} \to \RR_+$ for a batch sample of size $b\in [N]$ via the formula
$
\calL^b(\theta, A) 
\bydef
\frac{1}{2}\sum_{i\in A} \|Y_{i:} - \hat{Y}_{i:}\|^2. 
$
Therefore, an $\calL^b$-SGD sequence is any random sequence $(\theta_k^{\eta})_{k\in\NN}$ satisfying~\eqref{eq:GD} with $\Xi = [N]^{(b)}$ and an i.i.d. sequence $\xi_k\sim \operatorname{Unif}([N]^{(b)})$ for $k\in\NN_+$.
We stress that this construction corresponds to the usual mini-batch SGD.

Corollary~\ref{C:g_flow_convergence} states that, assuming enough overparametrization, the continuous trajectories given by the dynamics of the DI problem~\eqref{eq:dynamics} converge to the global minima of the loss~$\calL$ if the initial value $\theta_0$ is chosen properly, which happens WHP.  
In the theorem below we deduce an analogous convergence result for the $\mathcal{L}^b$-SGD iterates defined above.

\begin{theorem}\label{T:GD_main}
	Under Assumption~\ref{A:init_assumptions}, let $\beta^2_v=d_1^{-\rho}$ for some $\rho>0$ and $d_0,d_1,d_2,\delta_0,c(\delta_0),C(\delta_0)$ be as in Corollary~\ref{C:g_flow_convergence}. 
	Choose any error $\varepsilon > 0$, batch size $b = b(N)\in [N]$ and any family $\{ (\theta_k^\eta)\colon \eta>0\}$ of $\calL^b$-SGD sequences~\eqref{eq:GD}. 
	
	Then, there exists a step size $\eta_0\in (0,1)$ s.t. for a.e. $\eta\in(0,\eta_0)$,
	$
	\calL(\theta_{k^\ast}^\eta) < \varepsilon
	$
	for some 
	\begin{equation}\label{eq:k_ast_SGD}
	k^\ast \le 
	\Big\lfloor
	1 +  
	\frac{N}{\eta b}
	\max\big(
	0, 
	\frac{\log\left(
		{CN\log (N) d_0d_1\beta_w^2\beta_v^2}/{\varepsilon}
		\right)}{c(\delta_0)d_0d_1\beta^2_w}
	\big)
	\Big\rfloor,
	\end{equation}
	where $C>0$ is some absolute constant.
	The result holds with probability at least 
	$
	1 -
	\exp\big( 
	-
	\frac{d_0}{N}
	\cdot 
	\big[
	\frac{N^{2.5}}{d_0\beta_w^2} 
	\big]^{1/{(\rho+1)}}
	\big)
	-\calO(N^2)\exp( 
	-\Omega(N^{\delta_0/2}) 
	)
	-\exp(
	-\Omega(d_2\log N)
	).
	$
\end{theorem}
\begin{remark}
	Note that $k^\ast$ in Theorem~\ref{T:GD_main} depends on $\varepsilon$ via $\log(1/\varepsilon)$, i.e., SGD converges to the global minima at a linear rate.
\end{remark}

\begin{remark}\label{R:scaling}
	In order to compare the bounds obtained by Theorem~\ref{T:GD_main} with other works, one has to take into consideration not only parameters $\beta_w, \beta_v$ but also scaling of the data matrices $X$ and $Y$.
	E.g.,~\cite{oymak} works under the assumptions that $\|X_{:i}\| = 1$ for $i\in[N]$ and $\beta_w=1$, which by the properties of Gaussian distribution corresponds exactly to our case $\|X_{:i}\|=\sqrt{d_0}$ and $\beta_w=1/\sqrt{d_0}$.
\end{remark}
\begin{remark}\label{R:improvement}
	Corollary~\ref{T:GD_main} under the LeCun initialization, $\beta_w^2=1/d_0$, $\beta_v^2=1/d_1$, yields exponential loss convergence WHP for 
	$
	d_1 = 
	\tilde{\Omega}(N^{1.25}),
	$
	improving on $d_1 = \Omega(N^2)$ due to~\cite{nguyen}.
	Similarly, under different but equivalent scaling,~\cite[Corollary~2.4]{oymak} shows that overparametrization of the form 
	$
	d_1 = \Omega({N^4}/{d_0^3})
	$ 
	is sufficient for exponential loss convergence, when only the first layer is trained for $d_0 \in [\sqrt{N},N]$, whereas the second layer is fixed. 
	Neglecting the logarithmic factor, one can see that our bound $d_1 = 
	\tilde{\Omega}(N^{1.25})$ improves upon $d_1 = \Omega({N^4}/{d_0^3})$ for $\delta_0 \le 2.75/3\approx .92$, including practical datasets dimensions.
	Moreover, our bound works also for $\delta_0\in (0, 0.5)$ and for any $d_2$ (while they assume $d_2=1$).
	Finally, a simple adaptation of our technique combined with some observations from~\cite{oymak} allows to obtain the bound $d_1=\Omega(N^5/d_0^4)$ in training one layer setup, cf. Appendix~\ref{A:one_layer}.
\end{remark}

The main tool used to obtain Theorem~\ref{T:GD_main} is the following abstract result, which claims that under some technical conditions on $f$ and initialization scheme, the solutions to the DI involving $f$ are WHP close in the supremum norm to the trajectories of the corresponding piecewise interpolated processes.

\begin{theorem}[\cite{bianchi2020convergence}]\label{T:Bianchi_etal}
	For any probability space $(\Xi,\mathcal{F},\mu)$, let $f\colon\RR^{D}\times \Xi \to\RR$ be s.t. for some function $\kappa\colon\RR^D\times\Xi \to \RR_+$, the following conditions are satisfied:
	\begin{enumerate}[leftmargin=2em]
		\item $\forall\;x\in\RR^D,\; \exists\;\varepsilon>0,\; \forall\;z,y\in B(x,\varepsilon) ,\; \forall\;s\in\Xi, \; \| f(y,s)-f(z,s) \| \le \kappa(x,s) \| y-z \|$;\label{item:kappa_lipsch}
		\item $\forall\;x\in\RR^D,\; \exists\;K>0,\; \EE_{\xi\sim\mu}\kappa(x,\xi) \le K(1+\|x\|)$;\label{item:kappa_bounded}
		\item $\forall\;\substack{\mathcal{K}\subset\RR^{D}\text{ s.t.}\\\mathcal{K} \text{ is compact}}\;$, $\sup_{x\in\mathcal{K}}\EE_{\xi\sim\mu} \kappa(x,\xi)^2 < \infty$;\label{item:kappa_square_integrability}
		\item for a.e. $x\in\RR^D$, $f$ is $\mathcal{C}^2$ in some neighborhood of $x$.\label{item:f_C2}
	\end{enumerate}
	Then, for any time horizon $T>0$, the following DI problem is well-defined
	\begin{equation}\label{eq:DI_f}
	\dot{\theta}(t) \in - \partial \EE_{\xi\sim\mu} f(\theta(t), \xi)
	\; 
	\text{ for a.e. } 
	t\in[0,T].
	\end{equation}
	
	Moreover, if $\{\,(\theta_k^{\eta})_{k\in\NN_+}\colon\eta>0\,\}$ is a family of $f$-SGD sequences~\eqref{eq:GD} initialized at random continuously distributed~$\theta_0$, then there exists a set $\mathcal{N}\subset (0,\infty)$ s.t. $\mathcal{N}^c$ is of zero Lebesgue measure and s.t. for every compact set $\mathcal{K}\subset\RR^D$, time horizon $T>0$, and error $\tilde{\varepsilon}>0$,
	\[
	\lim_{\mathcal{N}\ni\eta\to 0^+}
	\PP\big(
	\exists\;\theta\colon[0,T]\to\RR^D
	\text{ solving~\eqref{eq:DI_f},}\;
	\theta(0) \in \mathcal{K},\;
	\sup_{t\in [0,T]} \vert \theta(t) - \bar{\theta}^{\eta}(t) \vert < \tilde{\varepsilon}
	\;
	\big\vert \; \theta_0\in\mathcal{K}
	\big)
	=
	1,
	\]
	where $\bar{\theta}^{\eta}$ is the corresponding random (measurable w.r.t. $(\theta_k^\eta)_{k\in\NN}$) piecewise interpolated process defined, i.e., 
	\begin{equation}\label{eq:interpolated_def}
	\bar{\theta}^{\eta}(t) 
	\bydef 
	\theta^{\eta}_k
	+ (t/\eta-k)(\theta^{\eta}_{k+1} - \theta^{\eta}_{k})
	\end{equation}
	for all $t\in [k\eta, (k+1)\eta), k\in\NN$.
\end{theorem}

The following theorem built upon \cite{bianchi2020convergence} can be seen as a general tool allowing to pass (when deducing global convergence) from the solutions to the DI~\eqref{eq:dynamics} to the SGD sequences given by~\eqref{eq:GD}.
We state it for general approximators (including, e.g., deep ReLU NN) and general loss functions as we believe it is of independent interest. 
In particular, we drop the assumption on the MSE loss and the NN denoted by $\hat{Y}$.
\begin{theorem}\label{P:abstract_convergence}
	Let $\tilde{\calL}_i \colon\RR^{D}\to\RR$ for $i\in[N]$ be arbitrary locally Lipschitz functions satisfying the chain rule~\eqref{eq:chain_rule} and being $\mathcal{C}^2$ in some neighborhood of a.e. point of $\RR^D$.
	Set $\tilde{\calL} = \sum_{i\in [N]}\tilde{\calL}_i$.
	Assume there exists a nonempty compact sets $Q\subset G\subset \RR^D$, s.t. any solution $\theta\colon [0,\infty)\to\RR^D$ to the DI
	\begin{equation}\label{eq:DI_prop}
	\dot{\theta}(t) 
	\in 
	-\partial\tilde{\calL}(\theta(t))
	\quad \forall\;t\ge 0,
	\end{equation}
	if initialized in $Q$, remains in $G$ and satisfies
	$
	\tilde{\calL}(\theta(t))
	\le 
	\tilde{\calL}(\theta(0))e^{-\gamma t}$
	for all  $t\ge 0$ and some  $\gamma>0$.
	Choose confidence threshold $\delta>0$, error $\varepsilon > 0$, batch size $b\in [N]$, and family $\{ (\theta_k^\eta)_{k\in\NN}\colon \eta>0 \}$ of $\tilde{\calL}^b$-SGD sequences given by~\eqref{eq:GD}, where $\Xi=[N]^b$, $\mu=\operatorname{Unif}([N
	]^b)$ and $\tilde{\calL}^b\colon \RR^D\times [N]^b\to \RR_+$ is given by $\tilde{\calL}^b(\theta, A) = \sum_{i\in A}\tilde{\calL}_i(\theta)$.
	Assume that $\theta_0$ is continuously distributed.

	Then, there exists a step size $\eta_0\in (0,1)$ s.t. for a.e. $\eta\in(0,\eta_0)$,
	$
		\PP(
		\tilde{\calL}(\theta_{k^\ast}^\eta) < \varepsilon
		\,\vert\, \theta_0\in Q)
		\ge 1-\delta
	$
	for 
	$
	k^\ast \le 
	\lfloor
	1 +  
	\frac{N}{\eta b}
	\max(
	0, 
	\gamma^{-1}
	\log(
	{2\varepsilon^{-1}\sup_{\theta\in Q}\tilde{\calL}(\theta)}
	)
	)
	\rfloor.
	$
\end{theorem}
\begin{proof}[Proof sketch of Theorem~\ref{P:abstract_convergence}]
	Let $l\bydef \sup_{\theta\in Q}\tilde{\calL}(\theta)$ and 
	\[
	T^\ast
	\bydef 
	\inf\{\, t\ge 0\colon  le^{-\gamma t} \le \varepsilon/2 \,\} 
	=
	\max\big(
	0, 
	\frac{\log(2l/\varepsilon)}{\gamma}
	\big) 
	\]
	so that all solutions to the DI~\eqref{eq:DI_prop} initialized in the set $Q$ fall to $\tilde{\calL}^{-1}([0,\varepsilon/2])$ before time $T^\ast$ (and clearly never escape it).
	Set $L\bydef \sup\{\,\| v\|\colon v\in \partial L(\theta), \theta\in G\,\}$.

	If we could apply Theorem~\ref{T:Bianchi_etal} with the family $\{(\theta_k^\eta)_{k\in\NN} \colon \eta>0\}$, $\tilde{\varepsilon}=\min(\varepsilon/2L,1)$ and $T=1+\frac{N}{b}T^\ast$, it would yield that for any $\delta\in(0,1)$, there exists $\eta_0\in(0,1)$ s.t. for a.e. $\eta\in (0,\eta_0)$,
	\[
		\PP\big(
			\exists\; \theta\;
			\text{solving}\; \dot{\theta}(t)\in -\partial\EE\tilde{\calL}^b(\theta(t),A_b) 
			\;\text{s.t. }
			\theta(0)\in Q
			\;\text{ and }\;
			\sup_{t\in [0, T]} \vert \theta(t)-\bar{\theta}^\eta(t)\vert < \tilde{\varepsilon}
			\;\big\vert\;
			\theta_0\in Q
		\big)
		\ge 
		1-\delta.
	\]
	Recall that $A_b\sim\operatorname{Unif}([N]^b)$ and note that $\EE \tilde{\calL}^b(\cdot, A_b) = \frac{b}{N}\tilde{\calL}(\cdot)$, whence if $\theta(t)$ solves $\dot{\theta}(t)\in -\partial\EE\tilde{\calL}^b(\theta(t), A_b)$, then $\theta(tN/b)$ solves~\eqref{eq:DI_prop}. 
	In particular $\tilde{\calL}(\theta(t))\le \varepsilon/2$ for any $t\ge \frac{N}{b}T^\ast$.
	Therefore, as for $\eta\in(0,\eta_0)$ it holds that $\frac{N}{b}T^\ast \le \eta k^\ast \le T$, then for a.e. $\eta\in(0,\eta_0)$,
	\[
		\tilde{\calL}(\theta_{k^\ast}^\eta)
		=
		\tilde{\calL}(\bar{\theta}^\eta(\eta k^\ast))
		\le 
		\tilde{\calL}(\theta(\eta k^\ast))
		+
		\tilde{\varepsilon}L
		\le 
		\varepsilon
	\]
	with probability at least $1-\delta$ conditioned on $\theta_0\in Q$.

	However, in general $\tilde{\calL}$ does not satisfy the assumptions of Theorem~\ref{T:Bianchi_etal}.
	In order to overcome this, we need to consider the set $G$ and modify $\tilde{\calL}$ outside of some neighborhood containing $G$, so that it becomes globally Lipschitz.
	As all solutions to~\eqref{eq:DI_prop} initialized in $Q$ remain in $G$, then it turns out that such modification does not conflict with the argument above, as is discussed in detail in~Appendix~\ref{A:pf_abstract_convergence}.
\end{proof}
We are ready to prove the main result of this section.
\begin{proof}[Proof of Theorem~\ref{T:GD_main}]
	For $\theta=(W,V)\in\RR^D$ and $\tilde{X}\in \RR^{N\times d_0}$, let 
	$\alpha_0(\tilde{X}, \theta)\bydef \sigma_{min}(\phi(\tilde{X}W)^T)$ 
	and 
	$\calL(\tilde{X}, \theta) \bydef \frac{1}{2}\|Y-\phi(\tilde{X}W)V\|_F^2$.
	Define
	\begin{align*}
	Q(\tilde{X}) 
	\bydef
	\{\,
	\theta\in\RR^D\colon \;&
	F(\theta,\tilde{X},Y) < \tfrac{1}{8},\quad
	\alpha_0(\tilde{X}, \theta) \ge  
	\sqrt{c(\delta_0)d_0d_1}\beta_w,
	\\&
	\calL(\tilde{X}, \theta)\le  Cd_0d_1d_2\beta_w^2\beta_v^2N\log(N),\quad
	\|\theta\| \le C(\sqrt{d_0d_1}\beta_w+\sqrt{d_1d_2}\beta_v)
    \},
	\end{align*}
	where $c(\delta_0)$ is the same constant as in Theorem~\ref{T:alpha_0_init}, $F$ is defined as in Theorem~\ref{T:deterministic_convergence_guarantee},~\eqref{eq:conditions_original_params}, and $C>0$ is some big enough absolute constant such that 
	\[
	\PP(\theta_0\in Q(X)) 
	\ge 
	1-
	\exp\big( 
	-
	\frac{d_0}{N}
	\cdot 
	\left[
	\frac{N^{2.5}}{d_0\beta_w^2} 
	\right]^{1/{(\rho+1)}}
	\big)-\calO(N^2)\exp( 
	-\Omega(N^{\delta_0/2}) 
	)
	-\exp(
	-\Omega(d_2\log N)
	),
	\]
	which is possible in virtue of Theorem~\ref{T:alpha_0_init} and Lemma~\ref{L:init_loss}, cf. Proof of Corollary~\ref{C:g_flow_convergence}.
	For each $\tilde{X}$, let 
	$u \bydef u(\tilde{X}, \theta) = {2\sqrt{\calL}(\tilde{X}, \theta)}/{\alpha_0^2(\tilde{X}, \theta)}$ 
	and
	$
	U(\tilde{X})\bydef 
	\sup_{\theta \in {Q}(\tilde{X})}\{\,
	\sqrt{2}\|\tilde{X}\|_{op}\|\theta\| 
	u\cdot 
	e^{\sqrt{2}\|\tilde{X}\|_{op}u}\,\},
	$
	so that any solution to the DI $\dot{\theta}\in -\partial\calL(\theta)$, if initialized in ${Q}(\tilde{X})$, remains in the set ${G}(\tilde{X}) = B({Q}(\tilde{X}), U(\tilde{X}))$ in virtue of Lemma~\ref{L:W_increments_bound} (cf., Proof of Theorem~\ref{T:deterministic_convergence_guarantee}).
	Moreover, $U(\tilde{X})<\infty$ by compactness of ${Q}(\tilde{X})$, whence ${G}(\tilde{X})$ is compact. 
	
	For each $\tilde{X}$, apply Theorem~\ref{P:abstract_convergence} with $\tilde{\calL}(\cdot)=\calL(\tilde{X},\cdot)$, $Q={Q}(\tilde{X})$, $\gamma=\inf_{\theta\in {Q}(\tilde{X})}\alpha_0^2(\tilde{X},\theta)$, $\delta = \PP(\theta_0 \notin Q(\tilde{X}))$ and $G={G}(\tilde{X})$, to get that for some $\eta_0\in (0,1)$ and a.e. $\eta\in(0,\eta_0)$,
	\[
		\PP(\calL(\tilde{X}, \theta_{k^\ast}^\eta) < \varepsilon \,\vert\, \theta_0 \in Q(\tilde{X}))
		\ge 
		\PP(\theta_0 \in Q(\tilde{X})),
	\]
	where $k^\ast$ is as in Theorem~\ref{P:abstract_convergence} and whence bounded as in~\eqref{eq:k_ast_SGD} by the definition of $Q(\tilde{X})$.
	Note that $\eta_0$ depends on $\tilde{X}$ only.
	
	Using the inequality $\PP(A\,\vert\,B) \le \PP(A)/\PP(B)$, multiplying both sides by $\delta$, integrating w.r.t. the distribution of~$X$ and estimating $(1-\delta)^2\ge 1-2\delta$, we get that 
	\[
		\PP(\exists\,{\eta_0\in(0,1)}\;\text{s.t. for a.e.}\;{\eta\in(0,\eta_0)},\; \calL(\theta_{k^\ast}^\eta) < \varepsilon)
	\]
	is at least $1-2\PP(\theta_0\notin Q(X))$, as desired.
\end{proof}

\section{Numerical Experiments}\label{sec:num_exps}
We present some numerical results illustrating two training setups -- when both layers $(W,V)$ are trained and when $W$ is trained only, complementing the experiments from~\cite[Section~4]{oymak}.

\subsection{Setup}\label{sec:num_setup}
Data is generated per single experimental run as follows: $N=200$, rows of $X$ are i.i.d. from the unit sphere, $d_2=1$ and labels $Y$ are randomly chosen s.t. half are set to~$1$ and the other half to~$-1$. 
In the first training setup $W$ has i.i.d. $\mathcal{N}(0,1)$ entries and $V$ has i.i.d. $\mathcal{N}(0,1/d_1)$ entries.
In the second training setup $W$ is as before and $V$ is fixed -- half of the entries are $1/\sqrt{d_1}$ and half are $-1/\sqrt{d_1}$ as in~\cite{oymak}. 
In all of the experiments we vary $d_0, d_1$. The NNs are implemented within the Pytorch framework. We used the standard SGD optimizer (in fact, GD as the batch size is set to $200$) with momentum ($0.9$). The learning rate differs on the training setup and is set to $0.15$ ($W$ only training), or $0.002$ ($(W,V)$ training).
\subsection{Results}
Figure~\ref{fig:NN2conv} illustrates the probability of convergence towards a global minimum depending on the network configuration. 
The probability is approximated based on $10$ independent runs and $d_0,d_1$ grid $2$ spaced, the convergence criterion is $\|\hat{y}-y\|/\|y\|<2.5e-03$ as in~\cite{oymak}. 
Compared with~\cite{oymak}, there seems to be no difference between training setups in terms of convergence probability and it is supposed that the overparametrization $N/d_0$ is sufficient for the global SGD convergence.
In Figures~\ref{fig:NN2corners},~\ref{fig:NNcorners} we present the average number of numerical zeros (absolute values below $1e-08$) in the preactivation layer at convergence.
Our investigation reveals an SGD optimization bias in both setups toward global minima with positive number of zero preactivation neurons (i.e., ReLU non differentiability points). 
In fact, these seem to be points of intersection of several ReLU activation pattern regions, as there are many zeros found. 
Note the different scales of the two plots -- the $W$ only training setup results in order of magnitude more numerical zeros than in the case of $(W,V)$ training.
This in particular suggests that the training trajectories might cross many different ReLU regions and thus they would be far from the linear regime described in~\cite{cohen}.
Below, we investigate further this phenomenon.
 

\begin{figure}[t!]
\centering
\begin{subfigure}[t]{0.3\textwidth}
\includegraphics[trim={20 0 20 20},clip,width=0.9\linewidth]{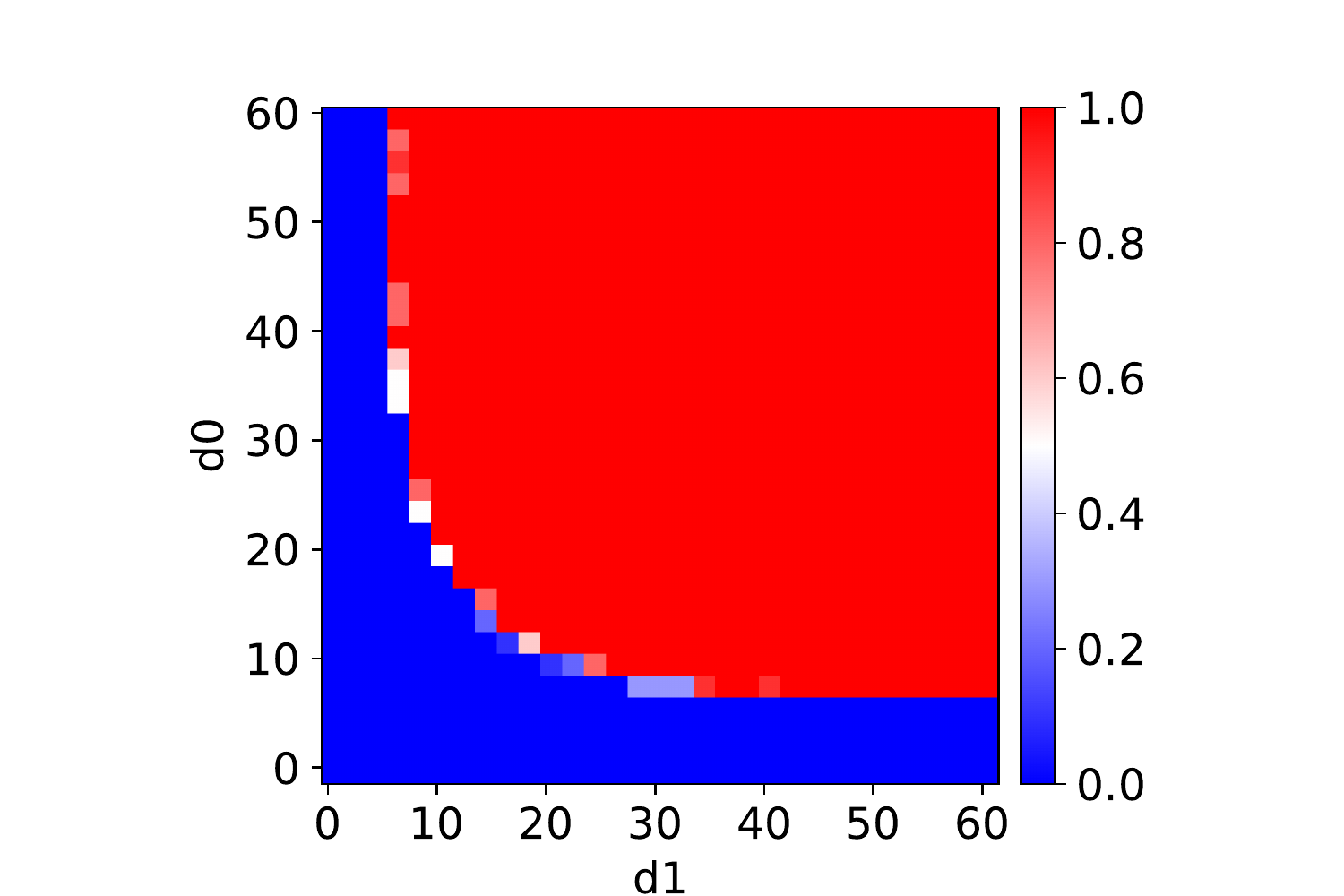}

\caption{Probability of convergence, $(W,V)$ training}
\label{fig:NN2conv}
\end{subfigure}
\begin{subfigure}[t]{0.3\textwidth}
\includegraphics[trim={20 0 20 20},clip,width=0.9\linewidth]{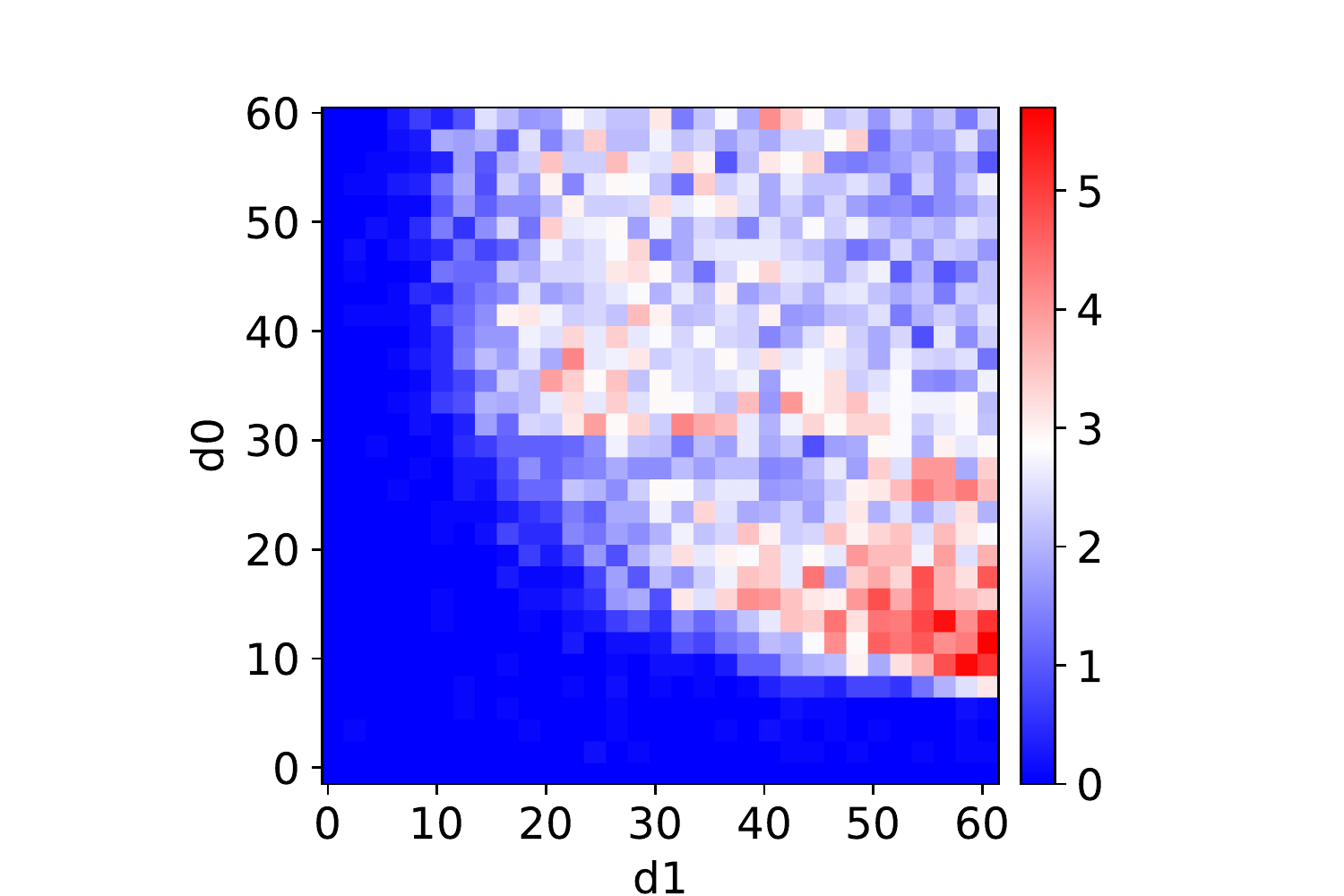} 
\caption{Avg. number of num. zeros in final preactivation, $(W,V)$ training}
\label{fig:NN2corners}
\end{subfigure}
\begin{subfigure}[t]{0.3\textwidth}
\includegraphics[trim={20 0 20 20},clip,width=0.9\linewidth]{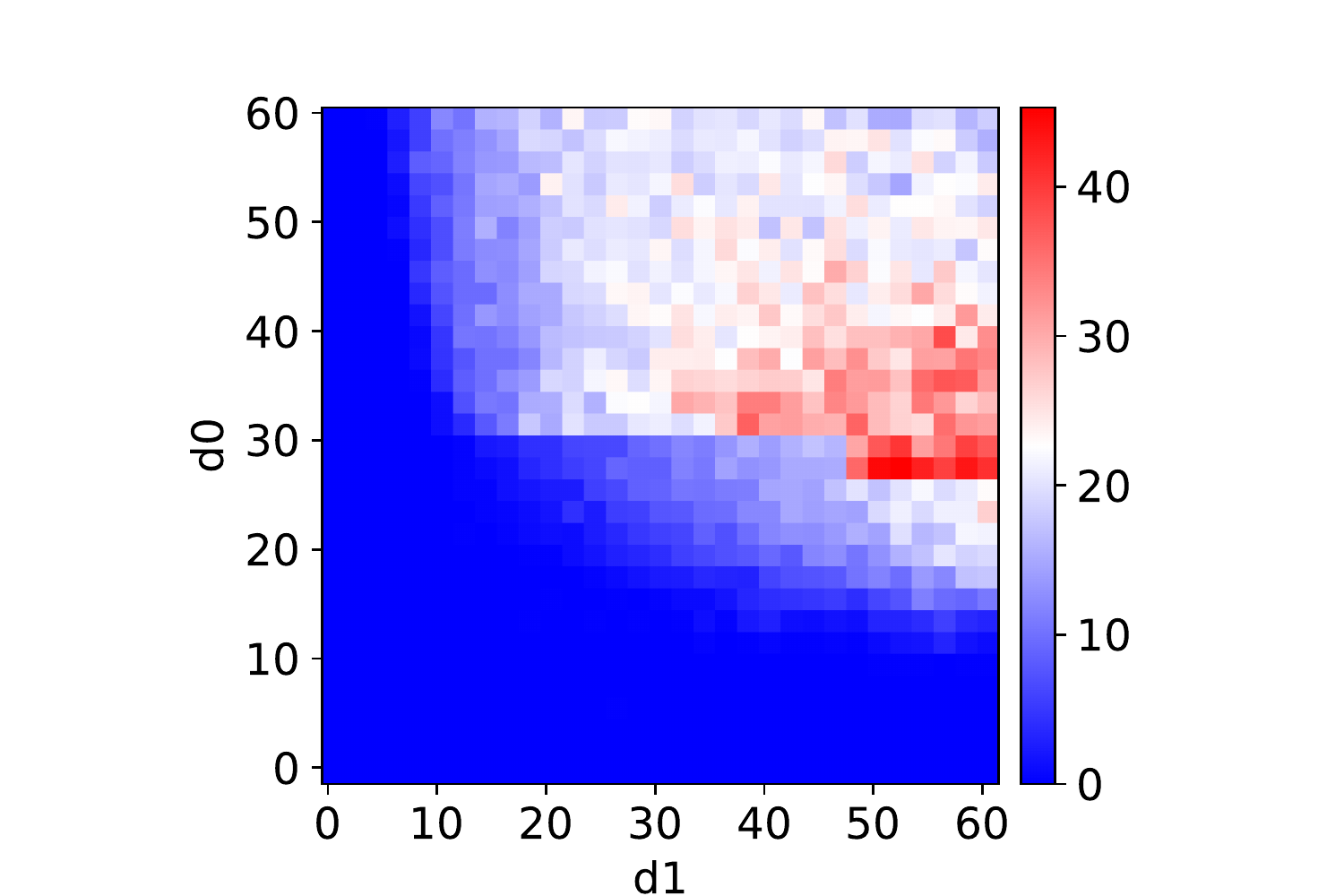} 
\caption{Avg. number of num. zeros in final preactivation, $W$ only}
\label{fig:NNcorners}
\end{subfigure}
%
%
\caption{Numerical results for both training setups after $50k$ SGD iterates.}
\vspace{-15pt}
\end{figure}

We now turn to Figures~\ref{fig:singlelayer} in which we analyze the training trajectories for both setups.
It is seen that despite being close to global minima (loss is already close to $0$ as seen on Figures~\ref{fig:loss1},~\ref{fig:loss2}), the number of numerical zeros in the preactivation pattern stays positive and is confined to a small range of values depending on the studied overparametrization level as presented on Figures~\ref{fig:zeros1},~\ref{fig:zeros2}. 
This confirms the observation above that the GD scheme prefers minima located close to the boundaries between several ReLU activation patterns. 
In fact, these seem to be corner points connecting several regions. 
We are not aware of any explanation of such a phenomenon in the literature. Moreover, despite being close to global minima, the activation patterns keep changing while performing the consecutive GD iterates before eventually stabilizing in some region. 
At which iteration that happens, depends on the overparametrization level as presented on Figures~\ref{fig:hamming1},~\ref{fig:hamming2}. This, in particular, demonstrates that most of the shallow ReLU networks training scheme happens in the nonlinear regime, i.e., it is not confined to a single ReLU activation region until the very end stage of training. 
The activation regions keep changing in a nonlinear fashion. Hence, the problem of studying the convergence of ReLU nets cannot be simplified to a study within a linear regime as suggested in~\cite{cohen}.

Finally, on Figures~\ref{fig:loss_change1},~\ref{fig:loss_change2},~\ref{fig:diff_norm1},~\ref{fig:diff_norm2} we investigated the relative loss change $\Delta\mathcal{L} = \frac{|\mathcal{L}(\theta_k) - \mathcal{L}(\theta_{k-1})|}{\mathcal{L}(\theta_{k-1})}$ and the relative differential change measured in the operator norm $\Delta D = \frac{\|D Y_k - D Y_{k-1}\|_{op}}{\|D Y_{k-1}\|_{op}}$. It is visible that the relative differential change is by order of magnitude larger than the relative loss change,  suggesting that the training for moderate and larger overparametrization levels is far from the lazy training regime studied in~\cite{lazy} characterized by $\Delta\mathcal{L}\gg \Delta D$.
\begin{figure}[h!]
	\centering
	%
	%
	\begin{subfigure}[t]{0.19\textwidth}
		\includegraphics[width=1\linewidth]{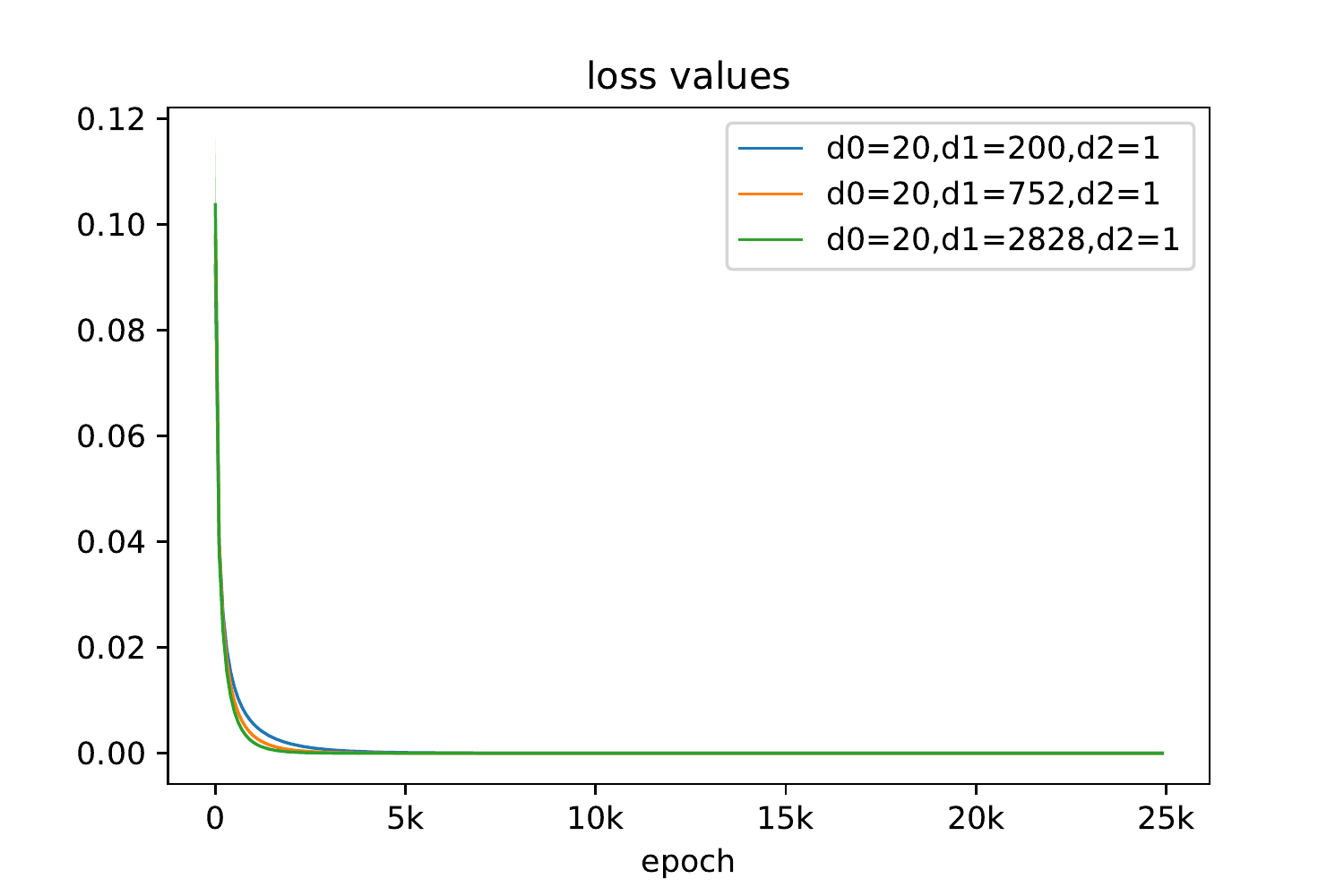} 
		\caption{Loss values}
		\label{fig:loss1}
	\end{subfigure}
	\begin{subfigure}[t]{0.19\textwidth}
		\includegraphics[width=1\linewidth]{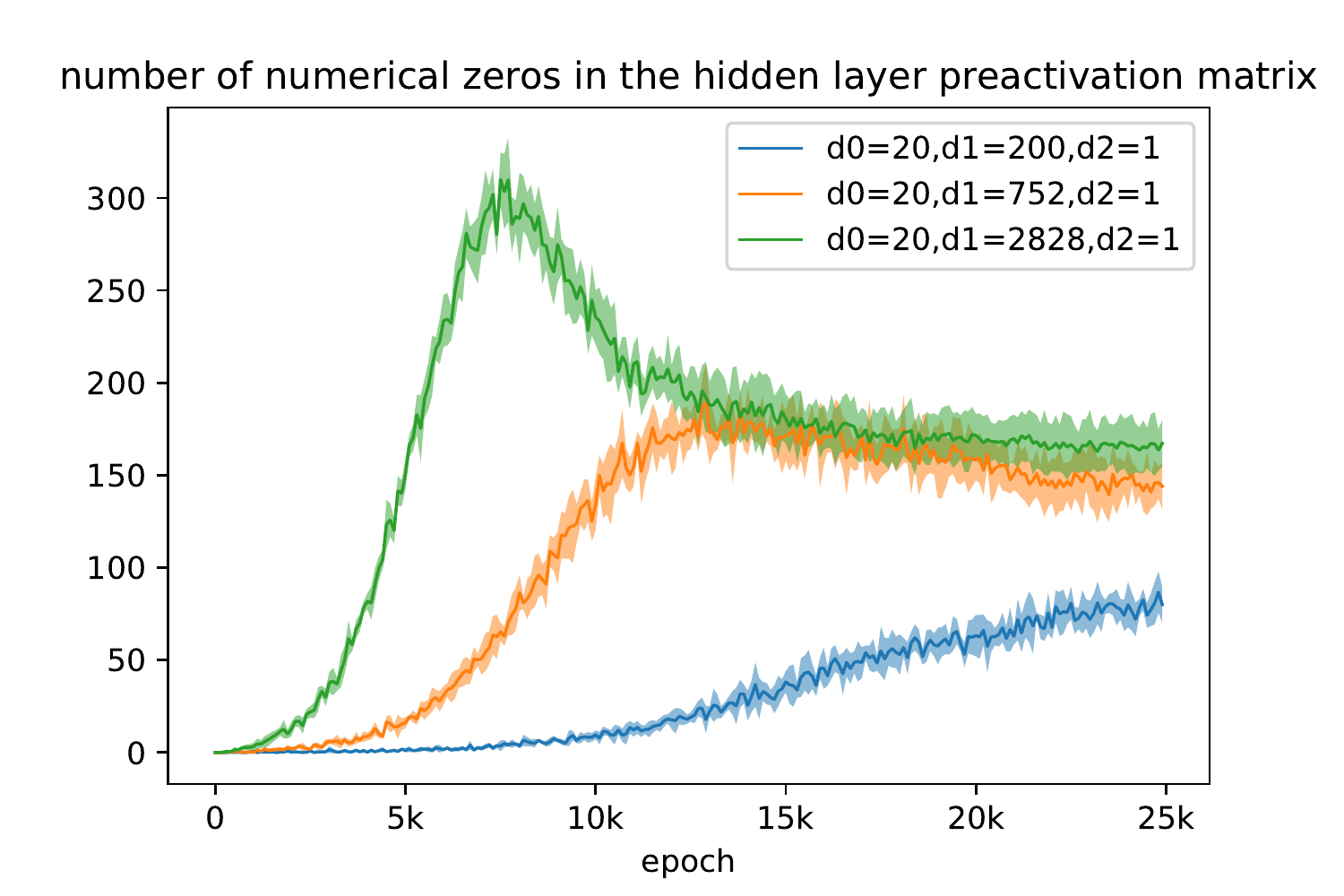} 
		\caption{Number of numerical zeros (threshold $1e-08$) in the preactivation}
		\label{fig:zeros1}
	\end{subfigure}
	\begin{subfigure}[t]{0.19\textwidth}
		\includegraphics[width=1\linewidth]{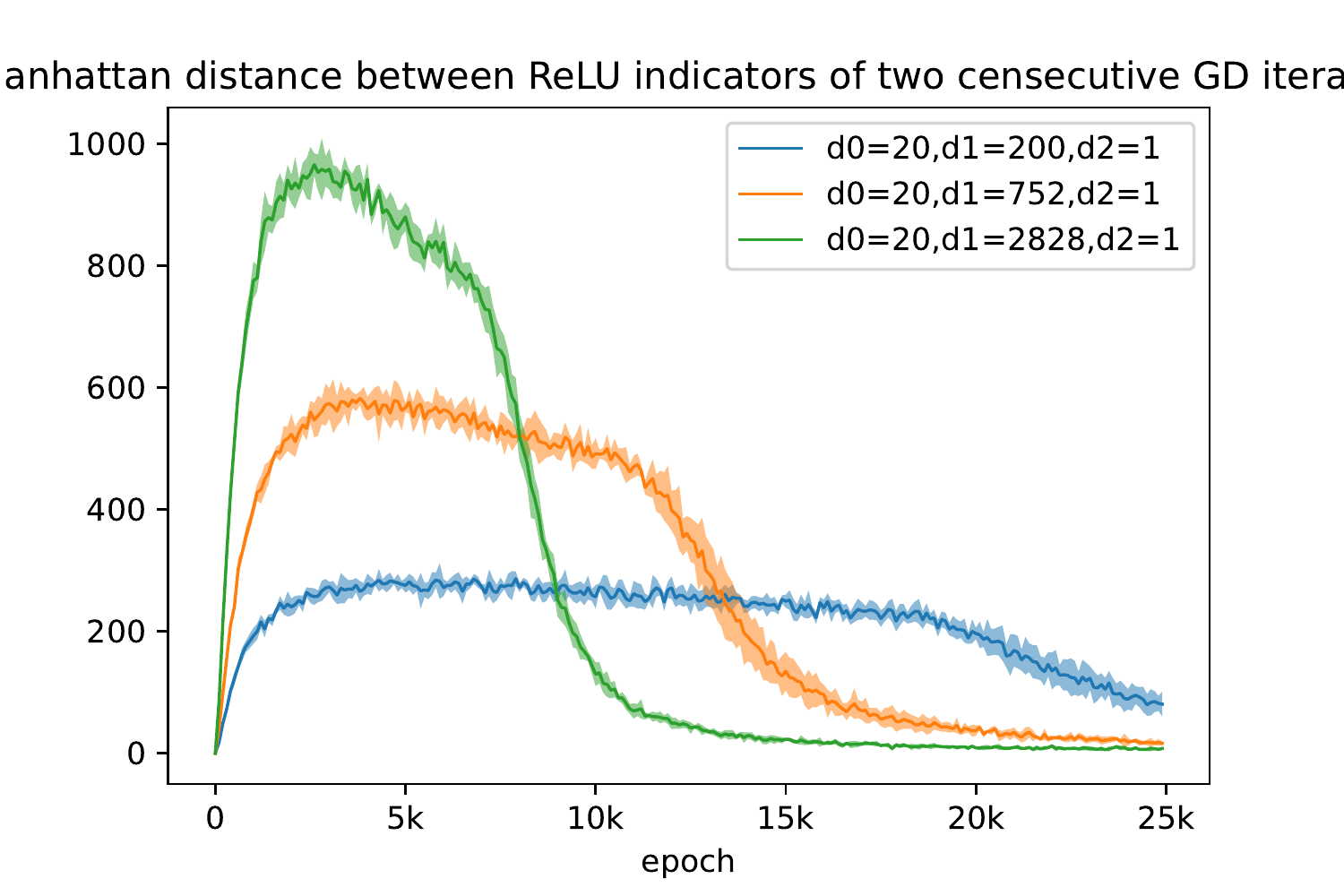} 
		\caption{Hamming distance between ReLU indicators of consecutive iterates}
		\label{fig:hamming1}
	\end{subfigure}
	\begin{subfigure}[t]{0.19\textwidth}
		\includegraphics[width=1\linewidth]{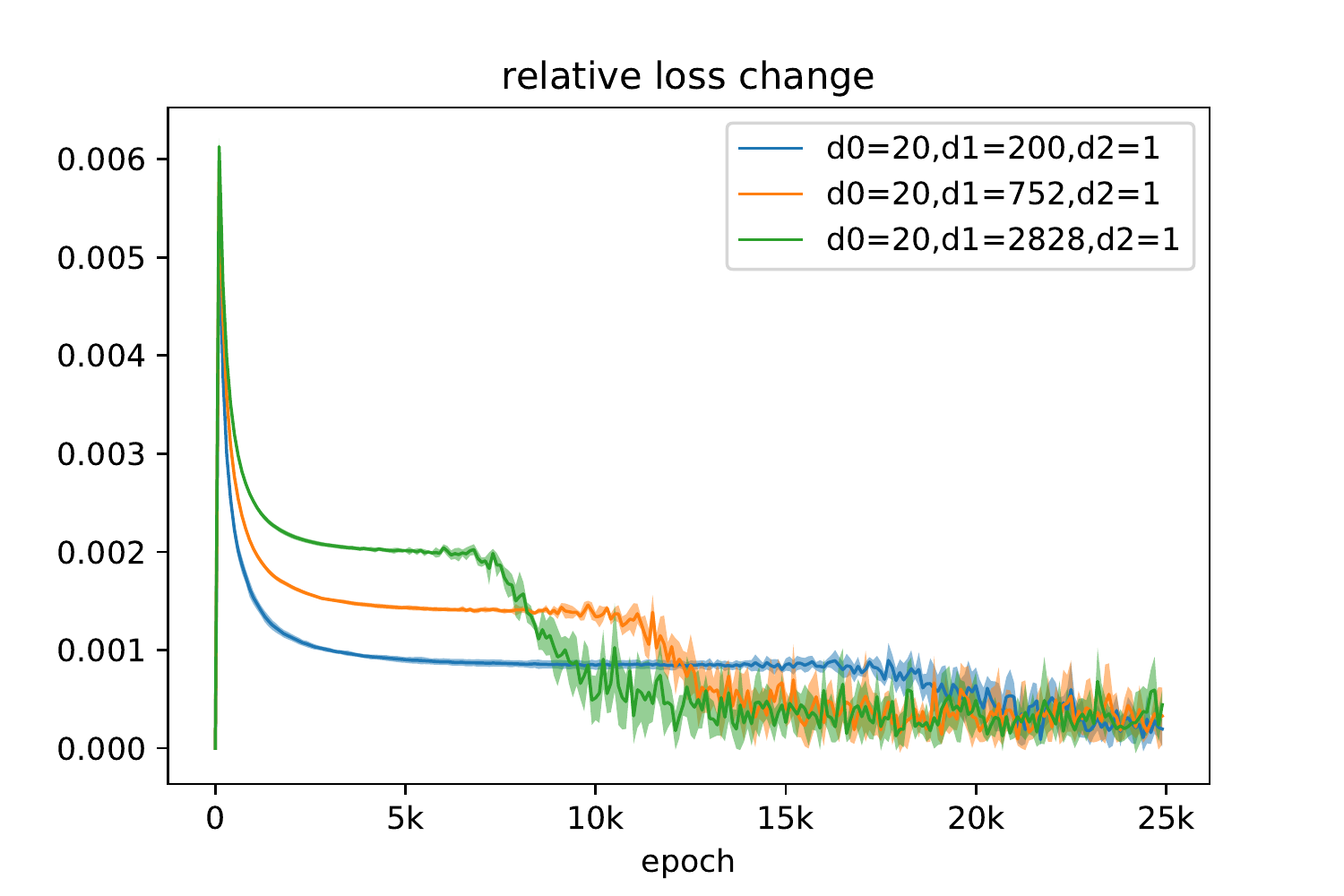}
		\caption{$\Delta \mathcal{L}$}
		\label{fig:loss_change1}
	\end{subfigure}
	\begin{subfigure}[t]{0.19\textwidth}
		\includegraphics[width=1\linewidth]{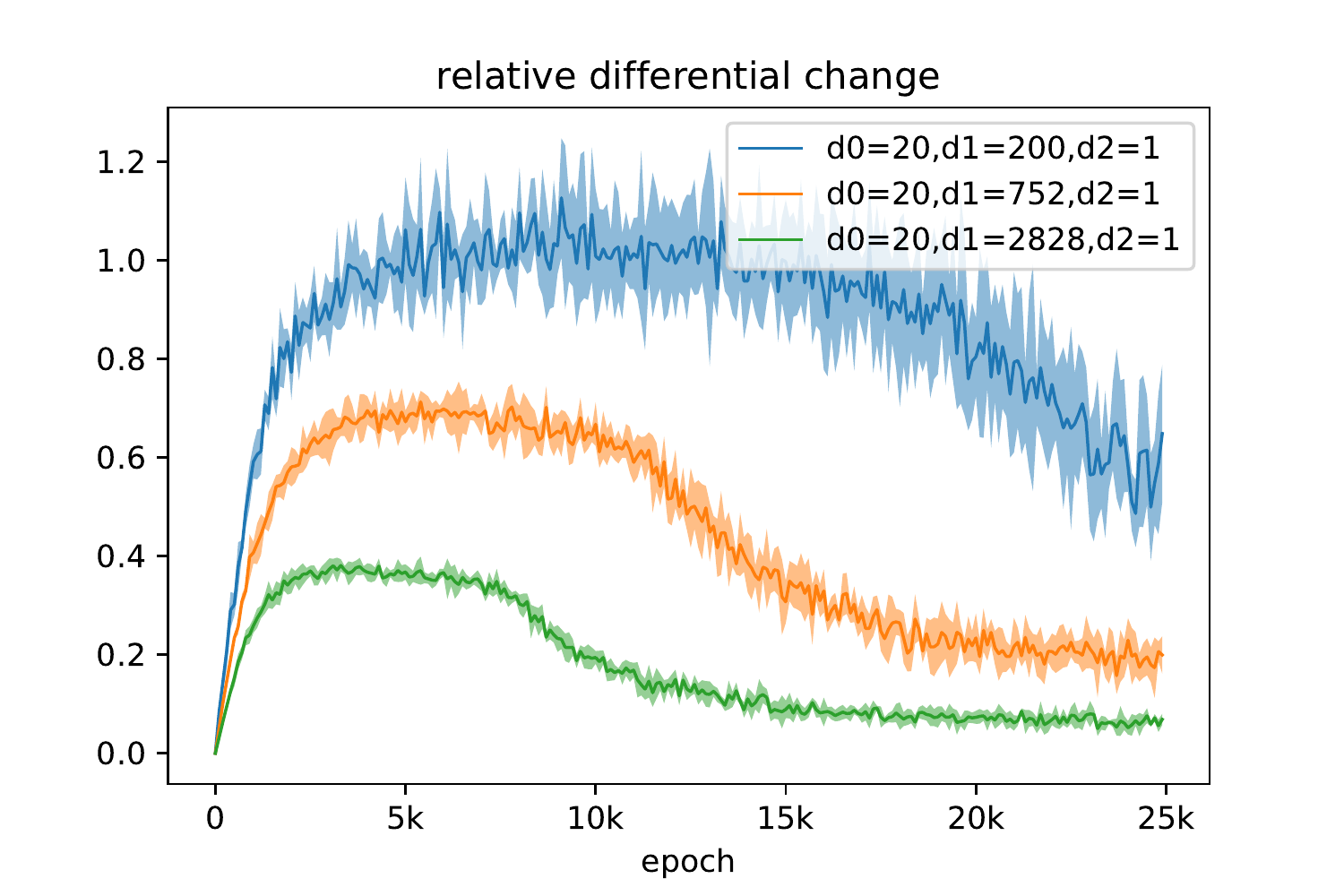}
		\caption{$\Delta Y$}
		\label{fig:diff_norm1}
	\end{subfigure}\\
\begin{subfigure}[t]{0.19\textwidth}
	\includegraphics[width=1\linewidth]{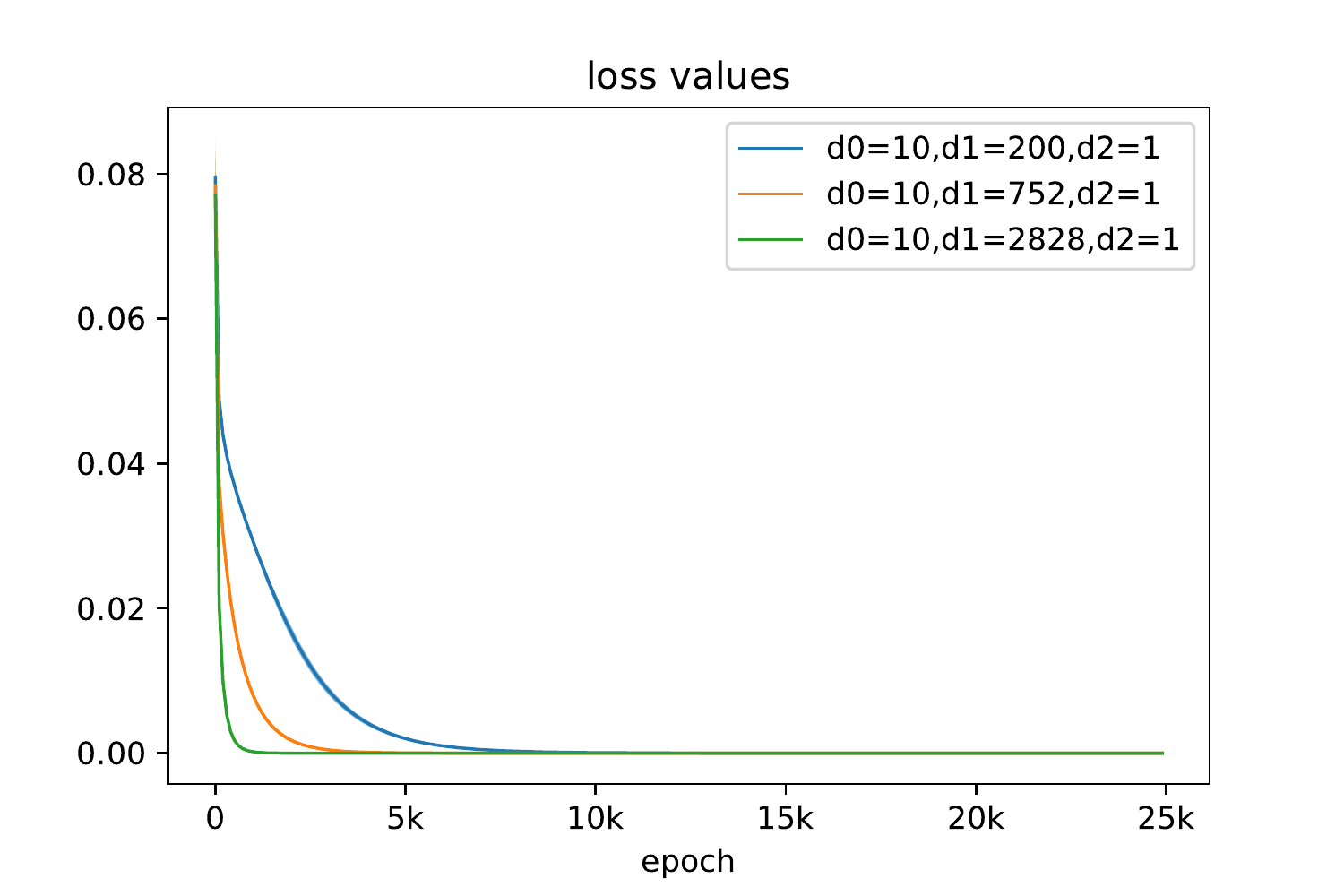} 
	\caption{Loss values}
	\label{fig:loss2}
\end{subfigure}
\begin{subfigure}[t]{0.19\textwidth}
	\includegraphics[width=1\linewidth]{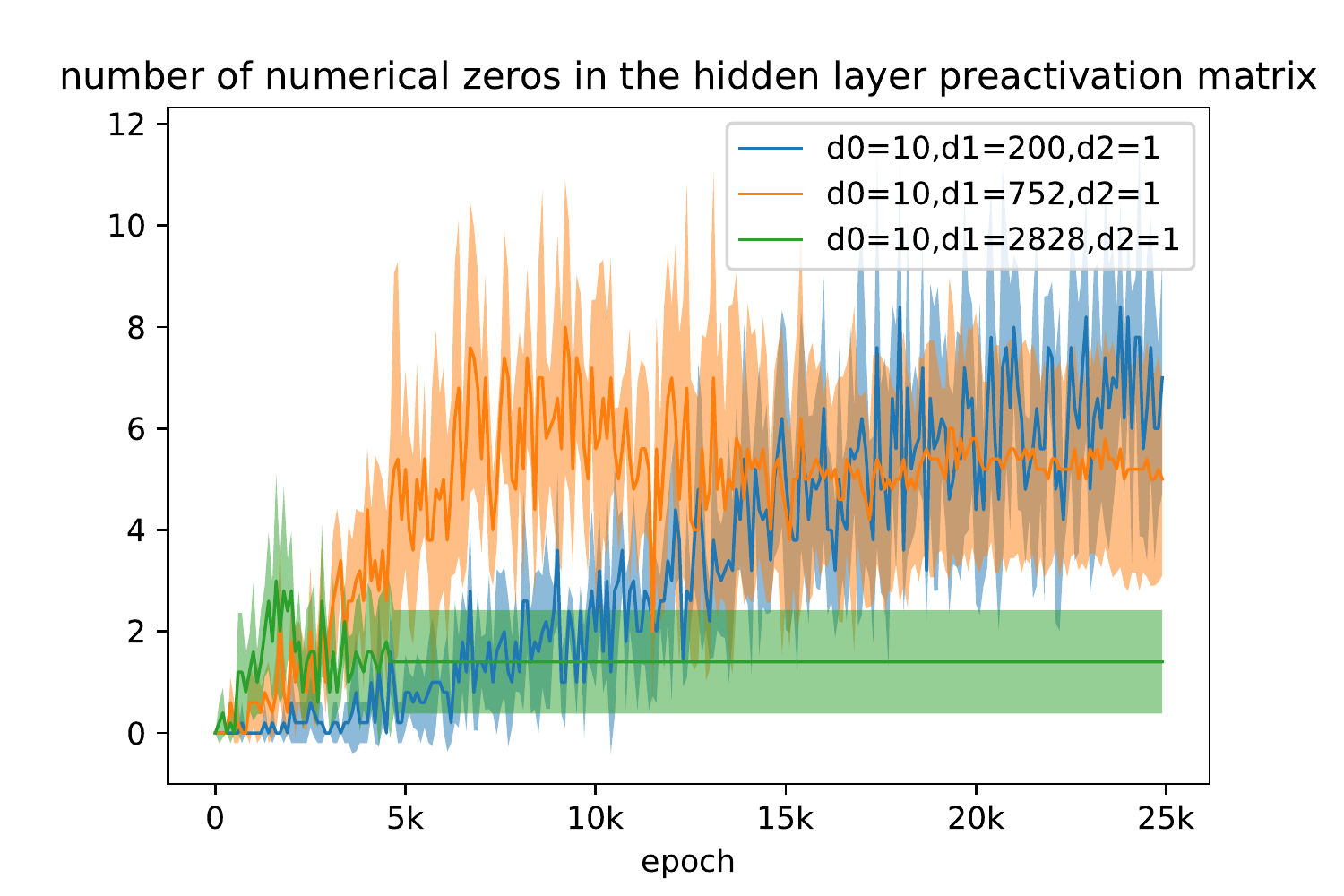} 
	\caption{Number of numerical zeros}
	\label{fig:zeros2}
\end{subfigure}
\begin{subfigure}[t]{0.19\textwidth}
	\includegraphics[width=1\linewidth]{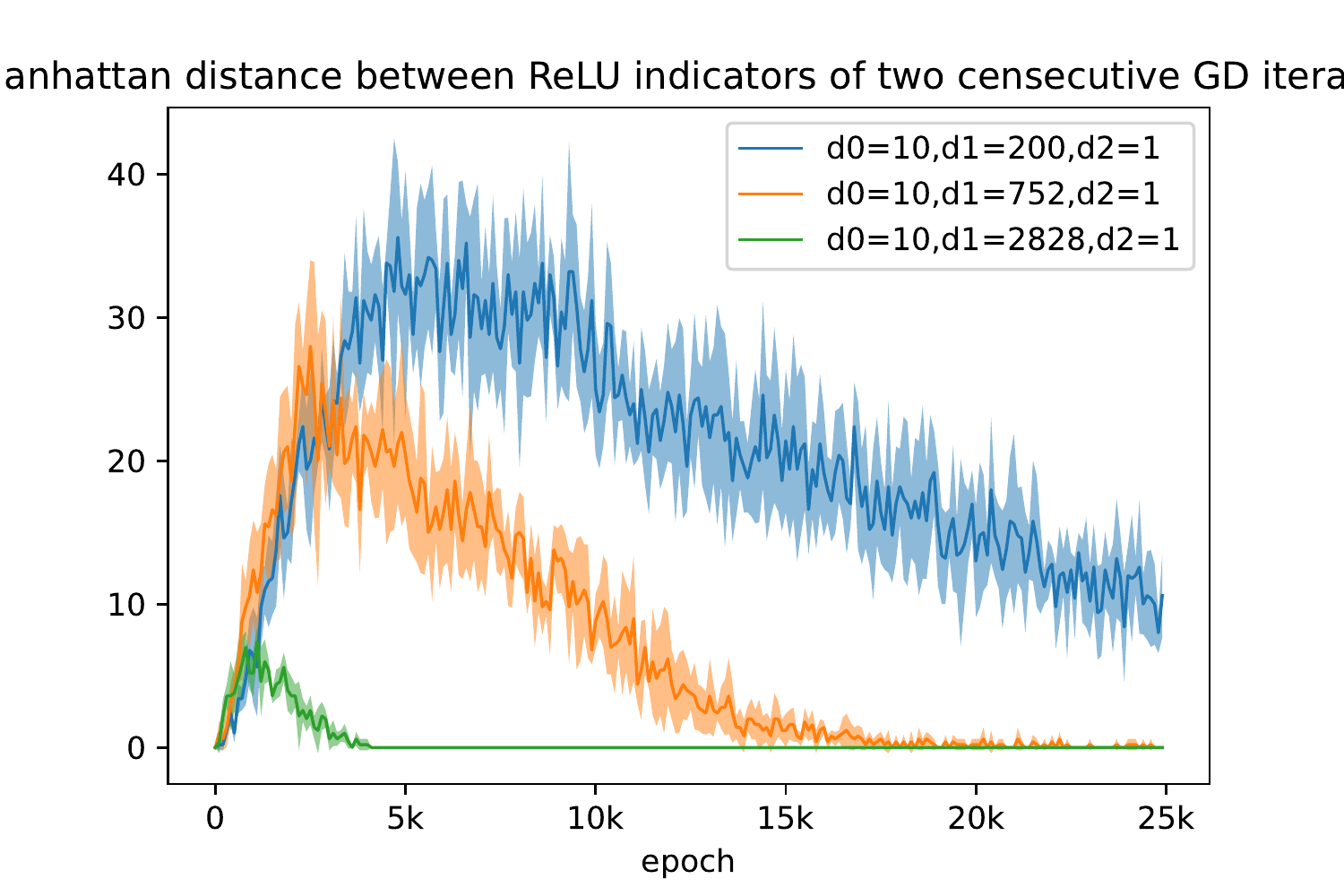}
	\caption{Hamming distance between ReLU indicators of GD iterates}
	\label{fig:hamming2}
\end{subfigure}
\begin{subfigure}[t]{0.19\textwidth}
	\includegraphics[width=1\linewidth]{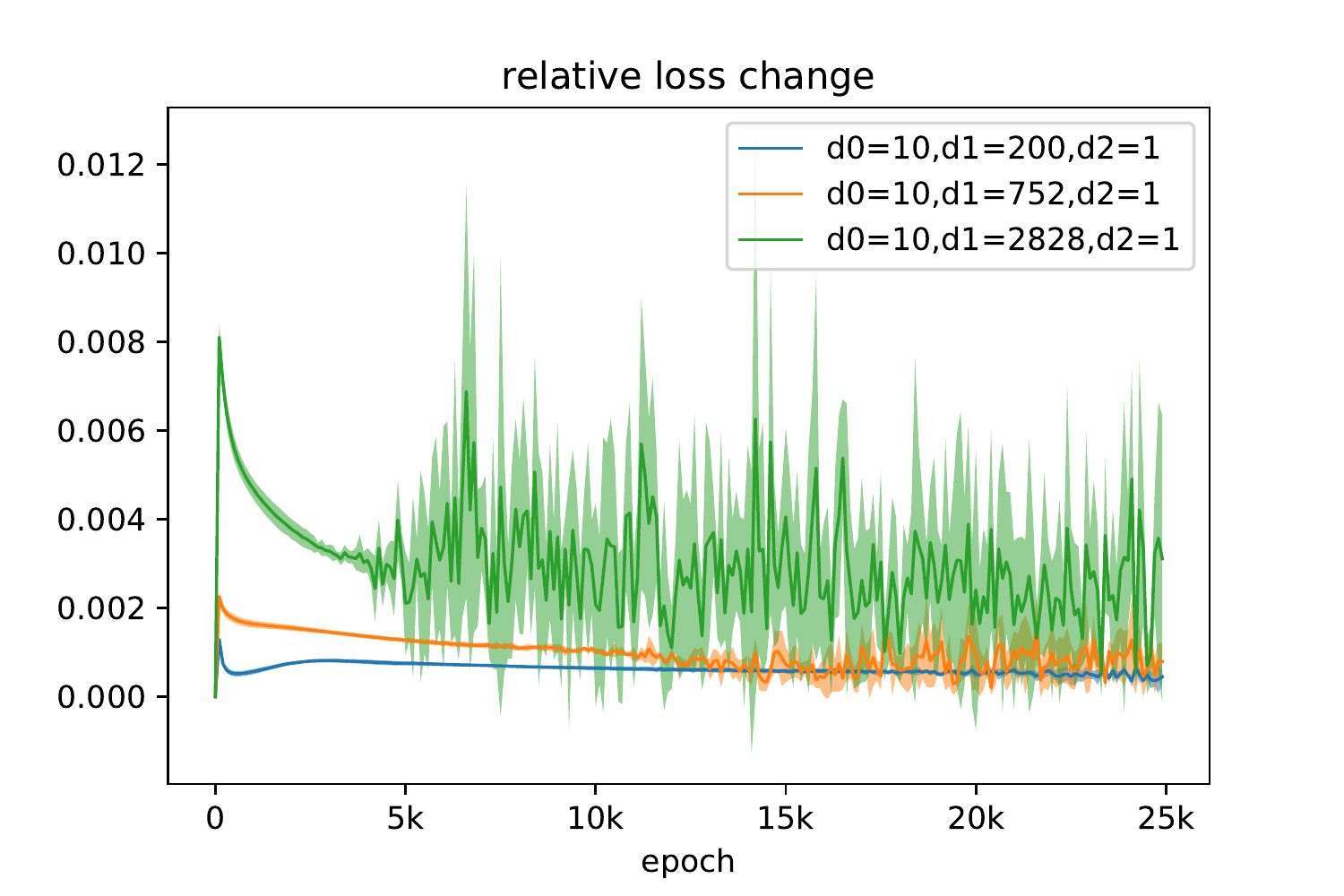}
	\caption{$\Delta \mathcal{L}$}
	\label{fig:loss_change2}
\end{subfigure}
\begin{subfigure}[t]{0.19\textwidth}
	\includegraphics[width=1\linewidth]{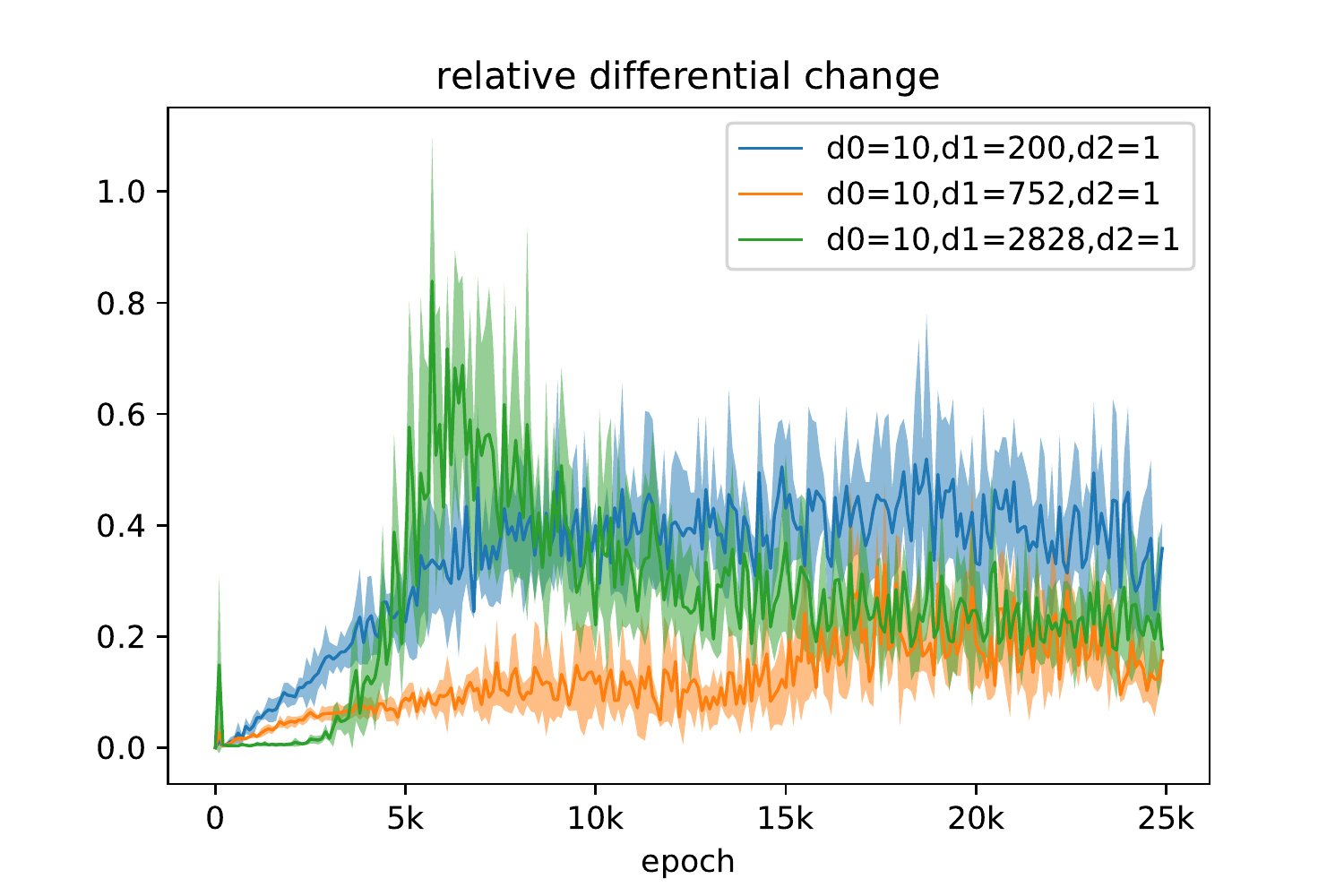}
	\caption{$\Delta Y$}
	\label{fig:diff_norm2}
\end{subfigure}
	\caption{Top row: training the hidden layer of the shallow NN only. Bottom row: training both layers of the shallow NN. The number of the hidden neurons is varied (the NN configuration is provided in the legend) and the total number of epochs of the evolution is equal to $25k$. The solid curve presents the mean from five independent runs, and the shaded region presents the standard deviation, plotted every 100th epoch.}
	\label{fig:singlelayer}
	\vspace{-15pt}
\end{figure}

\section{Conclusions and Future Work}\label{sec:conclusion}
We have demonstrated an improved trainability overparametrization bound of order $\tilde{\Omega}(N^{1.25})$ on the hidden layer of shallow NN equipped with ReLU activation functions.
We have obtained Theorem~\ref{P:abstract_convergence} -- an result allowing to pass from continuous solutions of the DI to the dynamics of SGD.
We believe that our contribution deepens the understanding of the optimization theory of~NN. There are several natural directions of further research and we list some of them below.
First direction is towards the theory of deep networks, where one could try to combine Theorem~\ref{P:abstract_convergence} with an analysis of DI dynamics in order to obtain improved overparametrization guarantees.
Secondly, Theorem~\ref{P:abstract_convergence} might serve as a tool to obtain overparametrization bounds which are suggested by numerical experiments in Section~\ref{sec:num_exps}.
Finally, all known bounds for ReLU NNs are valid under strong, probabilistic data assumptions and it would be of interest to pursue directions of research that would allow for more general data such as in the case of smooth activations,~cf. Table~\ref{tab:related_works}.

\bibliography{globconv_references}
\bibliographystyle{abbrvnat}
\pagebreak
\appendix
\section{Proof of Proposition~\ref{P:DI_existence}}\label{A:pf_existence}
We begin with a standard result on the existence of local solutions to the DI~\eqref{eq:DI_generic}, see, e.g.,~\cite[pp. 77-78]{Filippov1988} for a detailed proof. 
We refer the interested reader also to~\cite{aubin2012differential} for a comprehensive treatment of the theory of DIs.
\begin{theorem}[Existence of local solutions]\label{T:DI_solutions_existence}
	If $f\colon\RR^d\to\RR$ is locally Lipschitz, then there exists $T>0$ and an arc $x$ such that the differential inclusion~\eqref{eq:DI_generic} is satisfied for almost every $t\in[0,T)$.
	Moreover, for any bounded domain $G\ni x_0$, each solution to~\eqref{eq:DI_generic} lying within $G$ can be continued up until it hits the boundary of $G$.
\end{theorem}

Recall the definition of the chain rule~\ref{eq:chain_rule}.
Its importance is shown in the following lemma.
\begin{lemma}[\cite{davis2020}]\label{L:DI_solutions_property_chain_rule}
	If $f\colon\RR^d\to\RR$ is locally Lipschitz, satisfies the chain rule~\eqref{eq:chain_rule} and $x\colon\RR_+\to \RR^d$ is an arc satisfying the DI
	\[
	\dot{x}(t) \in
	-\partial f(x(t))
	\quad 
	\text{for a.e. } t\in [0,T)
	\]
	for some $T\in\RR_+\cup \{ \infty \}$, then the equality $\| \dot{x}(t) \| = \min\{\, \| v \| \colon\; v\in \partial f(x(t)) \,\}\bydef m(t)$ holds for a.e. $t\in [0,T)$ and in particular
	\[
	f(x(t)) = f(x(0)) 
	-
	\int_0^t m(s)\,ds
	\quad
	\text{for a.e. } t\in [0,T).
	\] 
\end{lemma}

Theorem~\ref{T:DI_solutions_existence} and Lemma~\ref{L:DI_solutions_property_chain_rule} combined with the observation that the MSE loss function $\calL$ satisfies the chain rule~\eqref{eq:chain_rule}, cf.~\cite{davis2020}, immediately yield the following result.
\begin{proposition}\label{P:DI_existence_only}
	For any initial point $\theta_0\in \RR^{D}$, there exists $T>0$ and a solution $\theta\colon [0,T)\to\RR^D$ to the DI~\eqref{eq:dynamics}. 
	Moreover, for any bounded domain $G\ni \theta_0$, each solution to~\eqref{eq:dynamics} can be extended up until it hits the boundary of $G$.
	Finally, for any such $\theta$,
	\begin{equation}\label{eq:dynamics_growth}
	\frac{d}{dt}\calL(\theta(t)) \le 
	- \min_{v\in \partial \calL(\theta(t))} \| v \|^2
	\quad 
	\text{for a.e. } t\in [0,T).
	\end{equation}
\end{proposition}

Recall the definition of Clarke subdifferential operator,
\begin{equation}\label{eq:Clakre-deff}
\partial f(x) 
\bydef
\conv \big\{\,
\lim_{n\to \infty} \nabla f(x_n)
\quad \text{for some}\quad 
D_f \ni x_n \overset{n\to\infty}{\to} x
\,\big\}
\quad
\forall\; x\in\RR^d,
\end{equation}
where $\conv$ is the convex hull operator.
The Clarke subdifferential satisfies $\partial f(x) = \{ \nabla f(x) \}$ for any $x\in D_f$.
Recall also that for a matrix $A$, we denote $i$-th row vector of $A$ by $A_{i:}$ and $i$-th column vector of $A$ by $A_{:i}$ and that for a vector $x$, $\dg(x)$ denotes the diagonal matrix with $x$ on the diagonal.

Lemma below provides a description of $\partial\calL(\theta)$ for general $\theta\in\RR^D$ of a single hidden-layer NN.
\begin{lemma}\label{L:derivatives}
	For any $\theta=(W,V)\in \RR^D$, set $R_{ij}(\theta)=R_{ij}(X_{i:}W_{:j})$ to be generalized gradient of the function $X_{i:}W_{:j}\mapsto \phi(X_{i:}W_{:j})$.
	Then
	\begin{align}
	\label{eq:W_dot}
	-\partial_{W_{:j}}\calL(\theta) 
	&=
	\big\{\,
	X^T\dg(r_{:j})(Y-\hat{Y})(V_{j:})^T
	\colon\;
	r_{ij} \in R_{ij}
	\text{ for } i\in[N]
	\,\big\}		
	\quad \text{for }
	j\in[d_1]
	,\\
	\label{eq:V_dot}
	-\partial_V\calL(\theta)
	&=
	\big\{\, H^T (Y-\hat{Y}) \,\big\},
	\end{align}
	where $H=\phi(XW)$.
\end{lemma}
\begin{proof}
	The formulas are immediate at points where $\mathcal{L}$ is differentiable, since in this case Clarke subdifferential coincides with the usual gradient.
	For points where $\mathcal{L}$ is not differentiable, we apply the definition~\eqref{eq:Clakre-deff}.
\end{proof}

\begin{proposition}\label{P:loss_dynamics}
	For any initial point $\theta_0\in\RR^D$ and any solution $\theta\colon[0,T)\to\RR^{D}$ to the DI~\eqref{eq:dynamics}, denote $\alpha_0(s) \bydef \sigma_{min}(H^T(\theta(s)))$.
	Then
	\[
	\calL(\theta(t)) \le 
	\calL(\theta(0))
	\exp
	\big(
	-2\int_0^t \alpha_0^2(s)\,ds
	\big)
	\quad
	\text{for a.e. } t\in[0,T).
	\]
\end{proposition}
\begin{proof}
	By Lemma~\ref{L:derivatives}, for any $\theta=(W,V)\in\RR^D$ and any $v\in\partial\calL(\theta)$ one has
	\[
	\| v \|^2
	\ge 
	\min_{u\in -\partial_V \calL(\theta)} \| u \|_F^2 
	\ge 
	\lambda_{min}(HH^T) \| (Y-\hat{Y}) \|_F^2.
	\]
	Therefore, by an application of Proposition~\ref{P:DI_existence}, $\frac{d}{dt}{\calL}(\theta(t)) \le - 2\alpha_0(t)\calL(\theta(t))$ for a.e. $t\in[0,T)$.
	The result follows in virtue of Gr\"{o}nwall's lemma.
\end{proof}
Finally, Proposition~\ref{P:DI_existence} is implied by Proposition~\ref{P:DI_existence_only} and Proposition~\ref{P:loss_dynamics} above.

\section{Proof of Lemma~\ref{L:W_increments_bound}}\label{A:pf_W_increments}
\begin{proof}
Observe that if $v\in\partial_{W}\calL(\theta')$ for some $\theta'=(W',V')\in\RR^D$, then in virtue of Lemma~\ref{L:derivatives}, $v=X^T\tilde{v}$ for some matrix $\tilde{v}$ for which it holds that for some choice of the generalized gradient $r_{ij}\in R_{ij}(\theta')$, for each $(i,j)\in [N]\times [d_1]$
\[
	\| \tilde{v} \|_F^2 = \sum_{j\in d_1} \| \dg(r_{:j}) (Y - \phi(XW')V')(V'_{j:})^T \|_F^2
	\le 
	\| (Y - \phi(XW')V') \|_F^2 \|V'\|_F^2,
\]
where we used Lemma~\ref{L:mtx_product_norm_inequality} and the fact that since $\phi$ is $1$-Lipschitz, then $R_{ij}\subset [-1,1]$ and so $\|\dg(r_{:j})\|_{op}\le 1$ for each $j\in [d_1]$. 
Whence by Lemma~\ref{L:mtx_product_norm_inequality}, $\| v \|_{F}\le \|X\|_{op}\| Y - \phi(XW')V' \|_F \| V' \|_F$.

Using the fact that $\theta$ is a solution to DI~\eqref{eq:dynamics}, passing with the norm under the integral and using the above estimate on $\| v\|_{F}$, we get that for any $t\in[0,T)$
\begin{align}\label{eq:W_increments_bound_proof}
\begin{split}
\| W(t) - W(0) \|_{F}
&=
\Big\| 
\int_{0}^{t} \dot{W}(s)\,ds
\Big\|_{F}
\\&\le
\int_{0}^{t} \| \dot{W}(s) \|_{F}\,ds
\\&\le 
\|X \|_{op} \int_{0}^{t} \| V(s) \|_F \sqrt{2\calL(\theta(s))}  \,ds
\\&=
\sqrt{2} \| X \|_{op} \Big(
\| V(0) \|_{F} \bar{\calL}(t)
+ 
\int_{0}^{t} \| V(s) - V(0) \|_{F} \sqrt{\calL(\theta(s))}\,ds
\Big).
\end{split}
\end{align}
Similarly
\begin{align}\label{eq:V_increments_bound_proof}
\begin{split}
\| V(t) -V(0) \|_{F}
&= 
\Big\|
\int_{0}^{t} \dot{V}(s)\,ds
\Big\|_F
\\&\le 
\int_{0}^{t}
\| \dot{V}(s)\|_{F} \,ds
\\&\le 
\int_{0}^{t} 
\| H(s) \|_{F} \| Y - \hat{Y}(s) \|_{F}\,ds
\\&\le 
\sqrt{2}\|X\|_{op}\Big(
\|W(0)\|_{F}\bar{\calL}(t)
+
\int_{0}^{t}\| W(s)-W(0) \|_{F} \sqrt{\calL(\theta(s))}\,ds
\Big).
\end{split}
\end{align}
Adding~\eqref{eq:V_increments_bound_proof} to~\eqref{eq:W_increments_bound_proof} and denoting $\tilde{\Delta}(s) = \| W(t) -W(0) \|_{F} + \| V(t) -V(0) \|_{F}$ gives
\[
\tilde{\Delta}(t) 
\le 
\sqrt{2}\|X\|_{op}\Big(
\big(\|W(0)\|_{F} + \|V(0)\|_{F}\big)
\bar{\calL}(t)
+
\int_{0}^{t}\tilde{\Delta}(s) \sqrt{\calL(\theta(s))}\,ds
\Big),
\]
whence {Gr\"{o}nwall's lemma and the triangle inequality yield~\eqref{eq:theta_increment_bound_result}.}

Let us turn to the proof of~\eqref{eq:W_increment_bound_result}.
Plugging~\eqref{eq:V_increments_bound_proof} into~\eqref{eq:W_increments_bound_proof} yields
\begin{multline}\label{eq:W_increments_intermediate}
\| W(t) - W(0) \|_{F}
\le 
\sqrt{2} \| X \|_{op}
\| V(0) \|_{F} \bar{\calL}(t))
+
2\|X\|_{op}^2 \|W(0)\|_F \int_0^{t}\bar{\calL}(s)
\sqrt{\calL(\theta(s))}\,ds\\
+
2\|X\|_{op}^2\int_0^{t}\sqrt{\calL(\theta(s))}
\int_0^{s}
\| W(u)-W(0) \|_{F} \sqrt{\calL(\theta(u))}\,du\,ds.
\end{multline}
Since for any function $f\colon\RR_+\to\RR$, and any $t\ge 0$, by Fubini's theorem
\[
\int_0^t f(s) \int_0^s f(u)\,du\,ds = \int_0^t\int_0^t f(s)f(u){\bf 1}_{u\le s}\,du\,ds = \frac{1}{2}\big(\int_0^t f(s)\,ds \big)^2,
\]
therefore 
\begin{equation}\label{eq:double_int_1}
\int_0^{t}\bar{\calL}(s)\sqrt{\calL(\theta(s))}\,ds = 
\frac{1}{2}\big( \bar{\calL}(t) \big)^2.
\end{equation}
Similarly, we can estimate
\begin{align}\label{eq:double_int_2}
\begin{split}
\int_0^{t}\sqrt{\calL(\theta(s))}
\int_0^{s}
\| W(u)-W(0) \|_{F} 
&\sqrt{\calL(\theta(u))}\,du\,ds
\\&\le 
\int_0^{t}\sqrt{\calL(\theta(s))}\,ds
\int_0^{t}
\| W(u)-W(0) \|_{F} \sqrt{\calL(\theta(u))}\,du
\\&= 
\bar{\calL}(t)
\cdot 
\int_0^t
\| W(s)-W(0) \|_{F} \sqrt{\calL(\theta(s))}\,ds.
\end{split}
\end{align}
Denote $\Delta(t)=\| W(t) - W(0) \|_{F}/\bar{\calL}(t)$. Estimating~\eqref{eq:W_increments_intermediate} with the use of~\eqref{eq:double_int_1} and~\eqref{eq:double_int_2}, and dividing by $\bar{\calL}(t)$ results in
\begin{equation*}
\Delta(t)
\le 
\sqrt{2} \| X \|_{op}
\| V(0) \|_{F}
+
\|X\|_{op}^2 \|W(0)\|_F\bar{\calL}(t)\\
+
2\|X\|_{op}^2
\int_0^{t}
\Delta(s)\bar{\calL}(s)\sqrt{\calL(\theta(s))}\,ds.
\end{equation*}
Therefore, by Gr\"{o}nwall's lemma
\begin{equation}\label{eq:W_increment_proof_final_bound}
\Delta(t) 
\le 
\Big(
\sqrt{2} \| X \|_{op}
\| V(0) \|_{F}
+
\|X\|_{op}^2 \|W(0)\|_F\bar{\calL}(t)
\Big)
\exp\Big(
2\|X\|_{op}^2
\int_0^t
\bar{\calL}(s)\sqrt{\calL(\theta(s))}\,ds
\Big).
\end{equation}
The conclusion follows by multiplying~\eqref{eq:W_increment_proof_final_bound} by $\|X\|_{op}\bar{\calL}(t)$ and using~\eqref{eq:double_int_1} in the exponent.
\end{proof}

\section{Proof of Theorem~\ref{T:alpha_0_init}}\label{A:pf_alpha_init}
Before we proceed to the proof of Theorem~\ref{T:alpha_0_init}, we introduce the necessary notation and auxiliary facts.
Since this theorem focuses on the properties of the initialization only, we write $W=W_0\in\RR^{d_0 \times d_1}$, $\alpha=\alpha(0)$ and $H=\phi(XW)$ for short (so that $\alpha=\sigma_{min}(H^T)$).
We introduce a random vector $w\sim \mathcal{N}(0,\Id_{d_0})$ which is independent of $X,W$, where $\Id_{d_0}\in\RR^{d_0\times d_0}$ is the identity matrix.
As $X$, $W$ and $w$ are independent, we denote by $\EE_X$, $\EE_W$ and $\EE_w$ the integration operators w.r.t. their respective laws.
Finally, recall that $\phi$ denotes the ReLU activation function and set
\[
G(X) \bydef 	
\EE_w [ \phi(Xw)\phi(Xw)^{T} ]
\quad \text{and}\quad 
\lambda(X) \bydef
\lambda_{min}(G(X)).
\]

Lemma~\ref{L:lambda_chernoff_bound} below shows how to control $\alpha$ with $\lambda(X)$ for a given matrix $X$.
We defer its proof until the end of this section.

\begin{lemma}[Lemma~5.2 in~\cite{nguyen2021tight}]\label{L:lambda_chernoff_bound}
	There exist some absolute constant $C>0$, s.t. for any $\Psi\colon \NN\to [1,\infty)$ and any $\tilde{X}\in\RR^{N\times d_0}$ satisfying $\lambda(\tilde{X})>0$, if
	\begin{equation}\label{eq:lemma_d1_overparam}
	d_1 \ge \max\Big(
	N,
	C\frac{\|\tilde{X}\|_{op}^2 (\Psi(N)+\log(N))}{\lambda(\tilde{X})}
	\cdot 
	\max\Big(
	1,
	\log\big(
	\frac{12\|\tilde{X}\|_{op}^2}{\lambda(\tilde{X})}
	\big)
	\Big)
	\Big),
	\end{equation}
	then
	\[
	\PP\Big(
	\alpha >  
	\frac{\beta_w\sqrt{d_1\lambda(X)}}{2}
	\;\Big\vert\;
	X=\tilde{X}
	\Big)
	\ge 
	1-\exp(-\Psi(N)).
	\]
	One can choose $C=16/(1-\log 2)$.
\end{lemma}

For any two matrices $A\in\RR^{n_r\times n_c}$, $B\in \RR^{n_r\times n_c'}$, 
recall that their Khatri-Rao product is defined as 
\[
A \ast B \bydef 
[A_{1:}\otimes B_{1:}, \ldots, A_{n_r:}\otimes B_{n_r:}]	
\in \RR^{n_r\times n_cn_c'}.
\]
Moreover, recall that any function $f\colon\RR\to\RR$, 
s.t. $\EE[f^2(g)] < \infty$ for $g\sim\mathcal{N}(0,1)$, admits the Hermite expansion, i.e.,
\[
\lim_{n\to\infty}
\EE\Big\vert
f(g)
-
\sum_{k=0}^{n} \mu_k(f)h_k(g)
\Big\vert^2 
= 0,
\]
where $h_k$ is the $k$-th probabilist's Hermite polynomial, i.e., 
\begin{equation}\label{eq:hermite-def}
	h_k(z) = (-1)^k e^{{z^2}/{2}}\frac{d^k}{dz^k}e^{-{z^2}/{2}},
\end{equation}
$h_0 \equiv 1$ 
and $\mu_k(f)$ is the $k$-th Hermite coefficient given by 
\begin{equation}\label{eq:hermite-coeffs}
	\mu_k(f) = \frac{\EE [f(g)h_k(g)]}{\sqrt{\EE h_k^2(g)}}.
\end{equation}

Lemma below follows from simple yet nontrivial calculations.
Since we were unable to find its proof, we provide it for completeness.
\begin{lemma}\label{eq:wiener_coeffs_decay}
	For any \emph{even} positive integer $k$,	
	\begin{equation}\label{eq:mu_r_bound}
			\vert \mu_k(\phi) \vert^2	
			=\Theta(k^{-2.5}).
	\end{equation}
\end{lemma}
\begin{proof}
We start with showing some properties of Hermite polynomials $h_k$.
By~\eqref{eq:hermite-def}, for any $k\in\NN$
\begin{equation}\label{eq:hermite-recurrence1}
	h_k'(z) 
	=
	zh_k(z)
	-
	h_{k+1}(z).
\end{equation}
Moreover, by induction it also holds that 
\begin{equation}\label{eq:hermite-recurrence2}
	h_k'=kh_{k-1}
\end{equation}
for $k\ge 1$. 
Indeed, the induction basis $h_1'=h_0$ is straightforward.
Assume that for some $k\in\NN$, $h_l' = lh_{l-1}$ for all $l\le k$.
Then, using~\eqref{eq:hermite-recurrence1}, induction assumption and~\eqref{eq:hermite-recurrence1} again, we obtain that
\begin{align*}
	h_{k+1}'(z) 
	&=
	(zh_k(z)-h'_k(z))'
	\\&=
	h_k(z)
	+
	zh_k'(z)
	-h_k''(z)
	\\&=
	h_k(z)
	+
	zkh_{k-1}(z)
	-kh_{k-1}'(z)
	\\&=
	h_k(z)
	+
	kh_k(z)
\end{align*}
as desired.
Combining~\eqref{eq:hermite-recurrence1} and~\eqref{eq:hermite-recurrence2} we obtain for any $k\in\NN_+$,
\begin{equation}\label{eq:hermite-free-term}
	h_{k}(0)
	=
	-h_{k-1}'(0)
	=
	-(k-1)h_{k-2}(0)
	=
	\begin{cases}
		(-1)^{k/2}(k-1)!! &\text{if}\quad 2\,\vert\,k,\\
		0 &\text{else}.
	\end{cases}
\end{equation}

We not turn to the calculation of $\mu_k(\phi)$.
Using integration by parts and~\eqref{eq:hermite-recurrence1} we obtain that for any $k\in\NN_+$,
\begin{align*}
\begin{split}
	\int_0^\infty 
	z
	e^{-{z^2}/{2}}
	h_k(z)
	\,dz
	&=
	-e^{-z^2/2} h_k(z) \Big\vert_0^\infty
	+
	\int_0^\infty 
	e^{-z^2/2}
	h_k'(z)
	\,dz
	\\&=
	h_k(0)
	+
	\int_0^\infty 
	e^{-z^2/2}
	h_k'(z)
	\,dz
	\\&=
	h_k(0)
	+
	\int_0^\infty 
	e^{-z^2/2}
	(zh_k(z)-h_{k+1}(z))
	\,dz,
\end{split}
\end{align*}
whence, after cancelling the same terms on both hand sides and using~\eqref{eq:hermite-free-term},
\begin{equation*}
	\int_0^\infty 
	e^{-z^2/2}
	h_{k+1}(z)\,dz
	=
	h_k(0)
	=
	\begin{cases}
		(-1)^{k/2}(k-1)!! &\text{if}\quad 2\,\vert\,k,\\
		0 &\text{else}.
	\end{cases}
\end{equation*}
Therefore,
\begin{align}\label{eq:hermite-first-part}
\begin{split}
	\int_0^\infty 
	z
	\frac{d^k}{dz^k}e^{-{z^2}/{2}}\,dz
	&=
	z
	\frac{d^{k-1}}{dz^{k-1}}e^{-{z^2}/{2}}
	\Big\vert_0^\infty 
	-
	\int_0^\infty 
	\frac{d^{k-1}}{dz^{k-1}}e^{-{z^2}/{2}}\,dz
	\\&= 
	z e^{-z^/2}h_{k-1}
	\Big\vert_0^\infty 
	-
	\int_0^\infty 
	\frac{d^{k-1}}{dz^{k-1}}e^{-{z^2}/{2}}\,dz
	\\&=
	-
	\int_0^\infty 
	\frac{d^{k-1}}{dz^{k-1}}e^{-{z^2}/{2}}\,dz
	\\&=
	(-1)^{k}
	\int_0^\infty 
	e^{-z^2/2} h_{k-1}(z)
	\,dz
	\\&=
	\begin{cases}
		(-1)^{\frac{k-2}{2}}(k-3)!! &\text{if}\quad 2\,\vert\, k,\\
		0 &\text{else.}
	\end{cases}
\end{split}
\end{align}
Similarly, for any $k\ge 1$, by integration by parts and by~\eqref{eq:hermite-recurrence2},
\begin{align}\label{eq:-hermite-second-part}
\begin{split}
	\EE h_k^2(g)
	&=
	\frac{1}{\sqrt{2\pi}}
	\int_\RR
	h_k(z)(e^{-z^2/2} h_k(z))
	\,dz
	\\&= 
	\frac{(-1)^k}{\sqrt{2\pi}}
	\int_\RR
	h_k(z)\frac{d^k}{dz^k}e^{-z^2/2}
	\,dz
	\\&= 
	\frac{(-1)^k}{\sqrt{2\pi}}
	\Bigl[
	h_k(z)\frac{d^{k-1}}{dz^{k-1}}e^{-z^2/2}\,\Big\vert_{-\infty}^\infty 
	-
	\int_\RR
	h_k'(z)\frac{d^{k-1}}{dz^{k-1}}e^{-z^2/2}
	\,dz
	\Bigr]
	\\&=
	\frac{(-1)^{k-1}}{\sqrt{2\pi}}
	\int_\RR
	h_k'(z)\frac{d^{k-1}}{dz^{k-1}}e^{-z^2/2}
	\,dz
	\\&=
	k
	\EE h_{k-1}^2(g)
	\\&= k!.
\end{split}
\end{align}
Therefore, combining~\eqref{eq:hermite-first-part} and~\eqref{eq:-hermite-second-part}, we obtain for positive even $k$,
\begin{align*}
\begin{split}
	\mu_k(\phi) 
	&=
	\frac{\EE [h_k(g)\phi(g)]}{\sqrt{\EE h_k(g)^2}}
	=
	\frac{1}{\sqrt{2\pi}}(-1)^{\frac{k-2}{2}}\frac{(k-3)!!}{\sqrt{k!}}.
\end{split}
\end{align*}
Using Stirling's formula, one obtains that for such $k$'s
\begin{align}
\begin{split}
	\vert \mu_k(\phi) \vert^2
	&=
	\Theta\Big(
		\frac{1}{k!}\cdot
		\frac{((k-2)!)^2}{((k-2)!!)^2}
	\Big)
	\\&= 
	\Theta\Big(
		\frac{1}{k^2}
		\cdot 
		\frac{(k-2)!}{((k/2-1)!)^22^{k-2}}
	\Big)
	\\&= 
	\Theta\Big(
		\frac{1}{k^2}
		\cdot 
		\frac{\sqrt{k-2}\bigl(\frac{k-2}{e}\bigr)^{k-2}}{(k/2-1)\bigl(\frac{k-2}{2e}\bigr)^{k-2}2^{k-2}}
	\Big)
	=\Theta(k^{-2.5})
\end{split}
\end{align}
as desired.
\end{proof}

Lemma below provides an interpretable lower bound on $\lambda(X)$ in terms of $X$.
We also defer its proof until the end of this section.
\begin{lemma}[Lemma~5.3 in~\cite{nguyen2021tight}]\label{L:lambda_Khatri-Rao}
	For any $r\in\NN_+$ and any non-zero $X\in\RR^{N\times d_0}$,
	\[
	\lambda(X) \ge 
	[\mu_r(\phi)]^2
	\frac{
		\lambda_{min}\big(
		(X^{\ast r})(X^{\ast r})^T
		\big)	
	}{
		\max_{i\in N}
		\| (X)_{i:} \|_2^{2(r-1)}.
	}
	\]
\end{lemma}

Recall that if a random vector $z\in\RR$ is sub-Gaussian, 
then there exists a constant 
$\sigma_z$ depending on $\|z\|_{\psi_2}$ only,
s.t. for any 1-Lipschitz function 
$f\colon\RR\to\RR$,
\[
\PP( z > \EE f(z) + t )
\le 
\exp\big(
-{t^2}/{\sigma_z^2}
\big),
\]
cf.~\cite[Proposition 2.5.2]{Vershynin_hdp}.
We are in position to prove Theorem~\ref{T:alpha_0_init}.

\begin{proof}[Proof of Theorem~\ref{T:alpha_0_init}]
	By Gershgorin circle theorem, for any $r\in\NN_+$,
	\begin{align}\label{eq:gershgorin}
	\begin{split}
	\lambda_{min}\big(
	(X^{\ast r})(X^{\ast r})^T
	\big)
	&\ge 
	\min_{i\in [N]}\| X_{i:} \|^{2r} - 
	N \max_{i\neq j} \vert 
	\langle X_{i:}, X_{j:} \rangle
	\vert^r
	\\&= 
	d_0^r - 
	N \max_{i\neq j} \vert 
	\langle X_{i:}, X_{j:} \rangle
	\vert^r.
	\end{split}
	\end{align}
	As for a fixed $y$,
	$\| \langle x,y\rangle \|_{\psi_2} \le \| y \| \cdot \| x \|_{\psi_2}$ 
	and as $\| X_{i:} \|_{\psi_2} = \calO(1)$ for $i\in[N]$, then for for any 
	$i,j\in[N]$, $i\neq j$
	\begin{align*}
	\PP\Big(
	\langle X_{i:}, X_{j:}\rangle > t
	\Big)
	= 
	\EE \Big[
	\PP\Big(
	\langle X_{i:}, X_{j:}\rangle > t
	\;\Big\vert\;
	X_{j:}
	\Big) \Big]
	\le 
	\exp\Big(
	-\frac{t^2}{\sigma_X^2d_0}
	\Big),
	\end{align*}
	where $\sigma_X>0$ is some absolute constant dependent on $\sup_i \| X_{i:} \|_{\psi_2}=\calO(1)$ only.
	Therefore, by the union bound
	\begin{align}\label{eq:intersetction_prob_lower_bd}
	\PP\Big(
	\bigcap_{i\neq j}
	\big\{
	\vert\langle X_{i:}, X_{j:}\rangle\vert < t
	\big\}
	\Big)
	\ge 
	1-N^2
	\exp\Big(
	-\frac{t^2}{\sigma_X^2d_0}
	\Big).
	\end{align}
	Whence, by~\eqref{eq:gershgorin} and Lemma~\ref{L:lambda_Khatri-Rao},
	\[
	\lambda(X)
	\ge 
	[\mu_r(\phi)]^2
	\frac{
		d_0^r
		-
		Nt^r
	}{
		d_0^{r-1}	
	}
	=
	[\mu_r(\phi)]^2d_0
	\big(1
	-
	Nt^rd_0^{-r}
	\big)
	\]
	holds with probability from~\eqref{eq:intersetction_prob_lower_bd} (at least).
	
	Recall that by assumptions of Theorem~\ref{T:alpha_0_init}, $d_0\ge N^{\delta_0}$ for some $\delta_0\in(0,1)$.
	Choosing $r=4\lceil \frac{1+\delta_0}{\delta_0}\rceil$, and using Lemma~\ref{eq:wiener_coeffs_decay}, reveals that 
	$
		[\mu_r(\phi)]^2
		= 
		\Omega(\delta_0^{2.5})
	$.
	Therefore, setting $t=d_0^{3/4}$ yields that for some absolute constant $\tilde{c}>0$
	\begin{align*}
	\lambda(X)
	\ge 
	[\mu_r(\phi)]^2d_0
	\big(1
	-
	Nd_0^{-r/4}
	\big)
	\ge 
	[\mu_r(\phi)]^2d_0
	\big(1
	-
	N^{-\delta_0}
	\big)
	\ge 
	\tilde{c}d_0\delta_0^{2.5}
	\end{align*}
	holds with probability at least $1-N^2\exp{(-N^{\delta_0/2}/\sigma_X^2)}$.
	Denote 
	\[
	A =
	\{\, \lambda(X) > \tilde{c}d_0\delta_0^{2.5} \,\}
	\cap 
	\{\, \|X\|_{op} \le \sqrt{DN} \,\}
	\]
	for some constant $D>0$ big enough and s.t. for some absolute constant $c>0$
	\[
	\PP(A) \ge 
	1-
	N^2\exp(-N^{\delta_0/2}/\sigma_X^2)
	- \exp(-cN),
	\]
	which is possible in virtue of~\eqref{eq:intersetction_prob_lower_bd} and Lemma~\ref{L:init_loss}.
	By conditioning and using Lemma~\ref{L:lambda_chernoff_bound}, we get that
	\[
	\PP\Big(
	\alpha > \frac{1}{2}\beta_w \sqrt{d_1 {\lambda}(X)}
	\Big)
	\ge 
	\PP(A)
	\big(
	1-\exp(-\Psi(N))
	\big)
	\]
	for any function $\Psi\colon\NN\to [1,\infty)$ provided that
	\[
	d_1 \ge \max\Big(
	N,
	C\frac{N(\Psi(N)+\log(N))}{\tilde{c}d_0\delta_0^{2.5}}
	\cdot 
	\max\Big(
	1,
	\log\big(
	\frac{12DN}{\tilde{c}d_0\delta_0^{2.5}}
	\big)
	\Big)
	\Big)
	\]
	for some (possibly different) absolute constant $C>0$.
	We conclude by noting that 
	\begin{multline*}
	N\max\Big(
	1,
	C\frac{(\Psi(N)+\log(N))}{\tilde{c}d_0\delta_0^{2.5}}
	\cdot 
	\max\Big(
	1,
	\log\big(
	\frac{12DN}{\tilde{c}d_0\delta_0^{2.5}}
	\big)
	\Big)
	\Big)
	\\ \le 
	N
	\max\Bigl(
	1, 
	C \delta_0^{-2.5}\log(1/\delta_0)
	\frac{\Psi(N)\log^2(N)}{d_0}
	\Bigr)
	\end{multline*}
	for some absolute constant $C>0$ and estimating
	\begin{align*}
	\PP(A)
	\big(
	1-\exp(-\Psi(N))
	\big)
	&\ge 
	1 - 
	N^2\exp{(-N^{\delta_0/2}/\sigma_X^2)}
	- \exp(-cN)
	- 
	\exp(-\Psi(N))
	\\&\ge 
	1 - 
	\calO(N^2)\exp{(-\Omega(N^{\delta_0/2}))}
	-
	\exp(-\Psi(N)).
	\end{align*}
\end{proof}
We proceed with proofs of the remaining lemmas.
\begin{proof}[Proof of Lemma~\ref{L:lambda_chernoff_bound}]
	Recall the definitions
	\[
	G(X) \bydef 	
	\EE_w [ \phi(Xw)\phi(Xw)^{T} ],
	\quad 
	\lambda(X) \bydef
	\lambda_{min}(G(X)),
	\]
	where 
	$w\sim\mathcal{N}(0,\Id_{d_0})$ 
	and introduce the following truncated versions of 
	$G$ and $\lambda$ for any $t>0$
	\[
	G_t(X) \bydef 	
	\EE_w [ \phi(Xw)\phi(Xw)^{T} {\bf 1}_{\| \phi(Xw) \| \le t } ],
	\quad 
	\lambda_t(X) \bydef
	\lambda_{min}(G_t(X)).
	\]
	As the map 
	$w\mapsto \|\phi(Xw)\|$ 
	is $\|X\|_{op}$-Lipschitz, then by Jensen's inequality, integration by parts and by Gaussian concentration
	\begin{align*}
	\|G_t(X) - G(X) \|_{op}
	&\le 
	\EE_w \big\|
	\phi(Xw)\phi(Xw)^{T} {\bf 1}_{\| \phi(Xw) \| > t }
	\big\|_{op}
	\\&=
	\EE_w \| \phi(Xw) \|^2  {\bf 1}_{\| \phi(Xw) \| > t }
	\\&=
	\int_0^{\infty}
	\PP\big(
	\|\phi(Xw) \|  {\bf 1}_{\| \phi(Xw) \| > t }
	\ge \sqrt{s}
	\;\vert\; X
	\big)
	\,ds
	\\&= 
	\int_0^{\infty}
	\PP\big(
	\|\phi(Xw) \|
	\ge \max(\sqrt{s},t)	
	\;\vert\; X
	\big)
	\, ds	
	\\&=
	t^2 
	\PP\big(
	\|\phi(Xw) \|
	\ge t
	\;\vert\; X
	\big)
	+
	\int_{t^2}^{\infty}
	\PP\big(
	\|\phi(Xw) \|
	\ge \sqrt{s}
	\;\vert\; X	
	\big)
	\, ds
	\\&\le 
	t^2\exp\Big(
	\frac{-t^2}{2\|X\|_{op}^2}
	\Big)	
	+
	\int_{t^2}^{\infty}
	\exp\Big(
	\frac{-s}{2\|X\|_{op}^2}
	\Big)\,ds
	\\&=
	\big(
	t^2	
	+
	2\|X\|_{op}^2
	\big)
	\exp\Big(
	\frac{-t^2}{2\|X\|_{op}^2}
	\Big)
	\le 
	6\|X\|_{op}^2
	\exp\Big(
	\frac{-t^2}{4\|X\|_{op}^2}
	\Big),
	\end{align*}
	where in the last line we have used the inequality $x\le e^{x}$ valid for all $x\in\RR$, with $x=t^2/4\|X\|_{op}^2$.
	Thus, for $X$ s.t. $\lambda(X)>0$, by choosing
	\begin{equation}\label{eq:t2_definition}
	t^2 = 
	4\|X\|_{op}^2	
	\max\Big(
	1,
	\log\Big(
	\frac{12\|X\|_{op}^2}{\lambda(X)}	
	\Big)
	\Big),
	\end{equation}
	we get that $\|G_t(X) - G(X) \|_{op} \le \lambda(X)/2$, whence in virtue of Weyl's inequality
	\begin{equation}\label{eq:lambda_t_estimate}
	\lambda_t(X) \ge \lambda(X)/2.
	\end{equation}

	Define 
	$H_t\in\RR^{N\times d_1}$ 
	via 
	$(H_t)_{:j} = \phi(XW_{:j}){\bf 1}_{\|\phi(XW_{:j})\| < t\beta_w}$ 
	for $j\in [d_1]$ 
	and observe that 
	$H_tH_t^T = \sum_{j\in [d_1]} (H_t)_{:j}(H_t)_{:j}^T$ 
	and 
	$\|(H_t)_{:j}(H_t)_{:j}^T \|_{op} = \| (H_t)_{:j}\|^2 \le \beta_w^2t^2$ 
	for all $j\in [d_1]$.
	By the matrix Chernoff inequality,~cf.~\cite[Theorem~1.2]{tropp2012tailbounds}, for any $\varepsilon\in(0,1)$
	\begin{equation}\label{eq:matrix_chernoff}
	\PP\Big(
	\lambda_{min}\big(
	H_tH_t^T
	\big) 
	\le 
	(1-\varepsilon)
	\lambda_{min}\big(\EE_W [H_tH^T_t]\big)
	\;\big\vert\;
	X
	\Big)	
	\\ 
	\le 
	N
	\Big[
	\frac{e^{-\varepsilon}}{(1-\varepsilon)^{1-\varepsilon}}
	\Big]^{\lambda_{min}\big(\EE_W [H_tH^T_t]\big) / \beta_w^2t^2}.
	\end{equation}
	Since for all $j\in[d_1]$, $W_{:j}\sim \mathcal{N}(0,\beta_w\Id_{d_0})$, then
	\begin{align*}
	\EE_W [H_tH_t^T] 
	&= 
	\sum_{j\in[d_1]} \EE_W [ (H_t)_{:j} (H_t)^T_{:j}]
	\\&=
	d_1 \EE_{w} \big[ \beta^2_w\phi(Xw)\phi(Xw)^T{\bf 1}_{\|\beta_w\phi(Xw)\|\le \beta_wt}\big]
	=
	\beta_w^2d_1G_t(X).
	\end{align*}
	Therefore, choosing $\varepsilon=1/2$ in~\eqref{eq:matrix_chernoff}, we obtain that
	\begin{equation}\label{eq:lambda_min_H_t_estimate}
	\PP\Big(
	\lambda_{min}(H_tH_t^T)
	\le 
	\frac{1}{2}\beta_w^2d_1\lambda_t(X)
	\;\big\vert\; X	
	\Big)
	\le 
	\exp\Big(
	-c\frac{d_1\lambda_t(X)}{t^2}
	+\log N
	\Big)
	\end{equation}
	for any $t\ge 0$,
	where $c=\frac{1}{2}(1-\log 2)$.
	
	Finally, noting that 
	$\alpha^2=\lambda_{min}(HH^T) \ge \lambda_{min}(H_tH_t^T)$ 
	and combining this observations with~\eqref{eq:lambda_t_estimate} and~\eqref{eq:lambda_min_H_t_estimate}, 
	we get that for any 
	$\tilde{X}$ 
	s.t. 
	$\lambda(\tilde{X}) > 0$,
	\[
	\PP\Big(
	\alpha \ge \frac{1}{2}\beta_w\sqrt{d_1 \lambda(X)}
	\;\vert\; X=\tilde{X}
	\Big)
	>
	1 - 
	\exp\Big(
	-c\frac{d_1\lambda(\tilde{X})}{2\tilde{t}^2}
	+ \log N
	\Big)
	\]
	for $\tilde{t}^2 = 4\|\tilde{X}\|_{op}^2	
	\max\big(
	1,
	\log\big(
	\frac{12\|\tilde{X}\|_{op}^2}{\lambda(\tilde{X})}	
	\big)
	\big)$.
	The conclusion follows as by~\eqref{eq:lemma_d1_overparam}
	\begin{align*}
	-c\frac{d_1\lambda(\tilde{X})}{2\tilde{t}^2}
	+ \log N
	= 
	\frac{-c d_1\lambda(\tilde{X})}
	{4\|\tilde{X}\|_{op}^2	
		\max\Big(
		1,
		\log\Big(
		\frac{12\|\tilde{X}\|_{op}^2}{\lambda(\tilde{X})}	
		\Big)
		\Big)}
	+ \log N
	\le 
	-\Psi(N).
	\end{align*}
\end{proof}

\begin{proof}[Proof of Lemma~\ref{L:lambda_Khatri-Rao}]
	Denote $J=\dg\big( \|X_{1:}\|,\ldots, \|X_{N:}\| \big)$ and $F=J^{-1}X$.
	Using the fact that for $g,g'\sim\mathcal{N}(0,1)$ and any $k,l\in\NN$, we get that for any $(i,j)\in[N]\times[d_1]$
	\begin{align*}
	J^{-1}G(X)_{ij}J^{-1}
	&= 
	\EE_w \big[
	\phi(\langle F_{i:}, w\rangle)
	\phi(\langle F_{j:}, w\rangle)
	\big]
	\\&=
	\lim_{n\to\infty}
	\EE_w \Big[
	\sum_{k=0}^{n} \mu_k(\phi) h_k(\langle F_{i:}, w\rangle )
	\sum_{l=0}^{n} \mu_l(\phi) h_l(\langle F_{j:}, w\rangle )
	\Big]
	\\&= 
	\lim_{n\to\infty}
	\sum_{k,l=0}^{n}
	\mu_k(\phi)\mu_l(\phi)
	\EE_w\Big[
	h_k(\langle F_{i:}, w\rangle )
	h_l(\langle F_{j:}, w\rangle )
	\Big]
	\\&= 
	\lim_{n\to\infty}
	\sum_{k=0}^{n}
	[\mu_k(\phi)]^2
	\langle 
	F_{i:}, F_{j:}
	\rangle ^k
	= 
	\sum_{k=0}^{\infty}
	[\mu_k(\phi)]^2
	\big\langle
	(F^{\ast k})_{i:},
	(F^{\ast k})_{j:}
	\big\rangle.
	\end{align*}
	Therefore, for a fixed $r\in\NN_{+}$,
	\begin{align*}
	\lambda(X)
	&=
	\lambda_{min}\Big(
	J \Big[
	\sum_{k=0}^{\infty}
	[\mu_k(\phi)]^2
	(F^{\ast k})(F^{\ast k})^T
	\Big] J
	\Big)
	\\&\ge 
	[\mu_r(\phi)]^2
	\lambda_{min}
	\Big(
	J (F^{\ast r})(F^{\ast r})^T J
	\Big) 
	\\&=
	[\mu_r(\phi)]^2
	\lambda_{min}
	\Big(
	J^{-(r-1)} (X^{\ast r})(X^{\ast r})^T J^{-(r-1)}
	\Big) 
	\\&\ge 
	[\mu_r(\phi)]^2
	\frac{\lambda_{min}\big( (X^{\ast r})(X^{\ast r})^T \big)}{\max_{i\in [N]} (J_{ii})^{2(r-1)}}
	\end{align*}
	as desired.
\end{proof}

\section{Proof of Lemma~\ref{L:init_loss}}\label{A:pf_lem_concentration}
\begin{proof}
	The concentration results for $W_0$, $V_0$ and $X$ are standard, cf~\cite[Theorems 3.1.1, 4.6.1]{Vershynin_hdp}.

	We turn to the concentration result for $\calL$.
	Let us write $V=V_0$ and $W=W_0$ for short.
	Since 
	$
	\calL(\theta_0)
	\le 
	2\|Y\|_F^2 + 2\|\hat{Y}\|_F^2
	$,
	then it suffices to estimate $\| \hat{Y}\|_F = \| \phi(XW)V \|_F$.
	
	As for any $A\in\RR^{N\times d_1}$, the function $\RR^{d_1\times d_2}\ni V \mapsto \| AV \|_{F}$ is $\|A\|_{F}$-Lipschitz with respect to the Frobenius norm on $\RR^{d_1\times d_2}$, then Gaussian concentration applied to $V$ yields
	\begin{equation}\label{eq:gaussian_concentration}
	\PP\big(\| AV \|_{F} 
	\le \EE\| AV \|_{F}+t\big)
	\ge 
	1-\exp\big(-t^2/2\beta_v^2\|A\|^2_F\big)
	\quad
	\forall\; t\ge 0.
	\end{equation}
	By independence of the entries of $V$, $\EE \langle v, V_{:j} \rangle ^2 = \beta_v^2 \| v \|^2$ for any $v\in \RR^{d_1}$ and $j\in[d_2]$.
	Using this fact and Jensen's inequality we get that
	\begin{align*}
	\EE\|AV\|_F \le 
	\sqrt{\EE\|AV\|_F^2}
	&=
	\Big( 
	\sum_{i\in [N], j\in [d_2]} \EE (A_{i:}V_{:j})^2 
	\Big)^{1/2}
	\\&=	
	\Big( 
	\sum_{i\in [N], j\in [d_2]} \beta_v^2\| A_{i:} \|^2
	\Big)^{1/2}
	= 
	\beta_v\sqrt{d_2} \|A\|_{F},
\end{align*}	
	whence~\eqref{eq:gaussian_concentration} with $A=\phi(XW)$, $t=\sqrt{d_2}\beta_v\|A\|_F\sqrt{\log(N)}$ and conditioned on the variables $X$, $W$, implies that
	\[
	\|\hat{Y}\|_{F} 
	\le 2\beta_v\sqrt{d_2 \log(N)}\|\phi(XW)\|_{F}
	\]
	with probability at least $1-\exp(-d_2\log(N)/2)$.
	Estimating $\|\phi(XW)\|_F \le \| XW \|_F \le \|X\|_{op}\|W\|_{F}$ yields the desired estimate for $\calL$.
\end{proof}

\section{Proof of Corollary~\ref{C:g_flow_convergence}}\label{A:pf_g_flow_convergence}

\begin{proof}
	As $N^{\delta_0} \le d_0 \le N$, then by Lemma~\ref{L:init_loss}
	\begin{gather*}\label{eq:init_assymptotics}
	\|W_0\|_F = \Theta(\beta_w\sqrt{d_0d_1}),\quad 
	\|V_0\|_F = \Theta(\beta_v\sqrt{d_1d_2}),\quad
	\|X\|_{op} = \calO(\sqrt{N}),\quad\text{and}\\ 
	\calL(\theta_0) = \calO(N\log(N)\cdot d_0d_1d_2\beta_v^2\beta_w^2)
	\end{gather*}
	with probability at least 
	\[
	1 -
	2\exp(-\Omega(d_0d_1)) -
	2\exp(-\Omega(d_1d_2)) -
	\exp(-\Omega(N)) -
	\exp(-\Omega(d_2\log(N))).
	\]
	Using Theorem~\ref{T:alpha_0_init} with 
	\[
	\Psi(N)=
	\frac{d_0}{N}
	\cdot 
	\Big[
	\frac{d_2N^{2.5}}{d_0\beta_w^2} 
	\Big]^{1/{(\rho+1)}}
	\] 
	yields 
	\begin{align*}
	F(\theta_0, X, Y)
	&=
	\Big(
	2\sqrt{2}\frac{\sqrt{\calL(\theta_0)}\|X\|_{op}^2\|V_0\|_F}{\alpha_0^3(0)}
	+
	2\frac{\calL(\theta_0) \|X\|_{op}^3 \|W_0\|_{F}}{\alpha_0^5(0)}
	\Big)
	\exp\Big(
	\frac{4\|X\|^2_{op}\calL(\theta_0)}{\alpha_0^4(0)}
	\Big)
	\\&=
	\calO
	\Big(
	\frac{N\sqrt{N\log N}\sqrt{d_0}d_1d_2\beta_w\beta_v^2}{d_0^{1.5}d_1^{1.5}\beta_w^3}
	+
	\frac{N^{2.5}\log(N) d_0^{1.5}d_1^{1.5} d_2\beta_w^3\beta_v^2}{d_0^{2.5}d_1^{2.5}\beta_w^5}
	\Big)
	\exp\Big(
	\calO\big(
	\frac{N^2\log Nd_0d_1d_2\beta^2_w\beta^2_v}{d^2_0d^2_1\beta_w^4}
	\big)\Big)
	\\&= 
	\calO
	\Big(
	\frac{d_2N\sqrt{N\log N} \beta_v^2}{d_0\sqrt{d_1}\beta_w^2}
	+
	\frac{d_2N^{2.5}\log(N) \beta_v^2}{d_0d_1 \beta_w^2}
	\Big)
	\exp\Big(
	\calO\big(
	\frac{d_2N^2\log N \beta^2_v}{d_0d_1\beta^2_w}
	\big)\Big)
	\\&=
	\calO\Big(
	\frac{d_2N^{2.5}\log(N) \beta_v^2}{d_0d_1 \beta_w^2}
	\Big)
	\exp\Big(
	\calO\big(
	\frac{d_2N^2\log N \beta^2_v}{d_0d_1\beta^2_w}
	\big)\Big)
	\\&= 
	\calO\Big(
	\frac{1}{d_1^{1+\rho}} \cdot 
	\frac{d_2N^{2.5}\log(N)}{d_0\beta_w^2}
	\Big)
	\exp\Big(
	\calO\Big(
	\frac{1}{d_1^{1+\rho}}\cdot 
	\frac{d_2N^2\log N}{d_0\beta^2_w}
	\Big)\Big)
	\\&=
	\calO\Big(\frac{1}{\log N}\Big)
	\exp\Big(\calO\Big(
	\frac{1}{\sqrt{N}(\log N)^{1+2\rho}}
	\Big) \Big)
	\end{align*}
	with probability at least 	$
	1-
	\exp\big( 
	-
	\frac{d_0}{N}
	\cdot 
	\big[
	\frac{d_2N^{2.5}}{d_0\beta_w^2} 
	\big]^{1/{(\rho+1)}}
	\big)
	-\calO(N^2)\exp( 
	-\Omega(N^{\delta_0/2}) 
	)
	-\exp(
	-\Omega(d_2\log N)
	)
	=o(1).
	$ -- we use tha fact that $\delta_0\le 1$ together with the probability bound of Theorem~\ref{T:alpha_0_init}. 
	The result follows in virtue of Theorem~\ref{T:deterministic_convergence_guarantee}.
\end{proof}

\section{Proof of Theorem~\ref{P:abstract_convergence}}\label{A:pf_abstract_convergence}
\begin{proof}
	Let $l=\sup_{\theta\in Q}\tilde{\calL}(\theta)$ and 
	\[
	T^\ast
	=
	\inf\{\, t\ge 0\colon  le^{-\gamma t} \le \varepsilon/2 \,\} 
	=
	\max\left(
	0, 
	\frac{\log(2l/\varepsilon)}{\gamma}
	\right) 
	\] 
	so that all solutions to the DI~\eqref{eq:DI_prop} initialized in the set $Q$ fall to $\tilde{\calL}^{-1}([0,\varepsilon/2])$ before time $T^\ast$ (and clearly never escapes it).
	For any $r>0$, let $G_r = G + B(0,r)$ be the $r$-widening of $G$ and denote 
	\[
	L \bydef \sup\{\, 
	\| v \|\colon
	v\in \partial \tilde{\calL}(\theta),\;
	\theta\in G_{1}
	\,\}.
	\]
	As $G_1$ is compact and $\tilde{\calL}$ is locally Lipschitz, then $L<\infty$.
	Choose $\Delta = \min(\varepsilon/2L, 1)$ and let $\zeta\colon\RR^D\to [0,1]$ be a smooth function such that $\zeta_{\vert G_{\Delta}}\equiv 1$ and $\zeta_{\vert (G_{2\Delta})^{c}}\equiv 0$.
	Finally, for any $\theta\in\RR^{D}$ and $A\in [N]^{(b)}$, let $f(\theta, A)=\tilde{\calL}^{b}(\theta, A)\cdot \zeta(\theta)$ and note that for $\theta\in G_\Delta$, $\EE f(\theta, A_b) = \frac{b}{N}\tilde{\calL}(\theta)$, where $A_b\sim \operatorname{Unif}([N]^b)$.
	
	Further, for any $\eta>0$, let $\tau^\eta=\inf\{ k\in\NN_+ \colon \theta^\eta_k \notin G_\Delta \}$ and set $\chi_k^\eta = \theta_k^{\eta}$ for $k \in [\tau^\eta]$ (note that $\tau^\eta$ is a random variable measurable w.r.t. the sigma field generated by $\theta_0,\xi_1,\xi_2,\ldots$) and $\chi^\eta_k$ for $k > \tau^\eta$ to be arbitrary such that $\chi^\eta_{k} \in -\eta\partial f(\chi^\eta_{k-1}, \xi_k)$, where $\theta^\eta_k$ is the $\tilde{\calL}^b$-SGD sequence belonging to the family given in the statement of the theorem. 
	Since for any $A\in[N]^b$, $(f(\cdot, A))_{\vert G_\Delta} = (\tilde{\calL}^b(\cdot, A))_{\vert G_\Delta}$, then $(\chi^\eta_k)_{k\in \NN}$ is indeed an $f$-SGD sequence.
	Additionally, by construction $(\chi_k^\eta)_{k\in \NN}$ escapes the set $G_\Delta$ if and only if $(\theta_k^\eta)_{k\in \NN}$ does so.
	
	As for any $A\in [N]^{(b)}$, $f(\cdot, A)$ is locally Lipschitz and compactly supported, then it is Lipschitz with some constant $L_A<\infty$.
	Therefore $f$ satisfies assumption~\ref{item:kappa_lipsch} of Theorem~\ref{T:Bianchi_etal} with $\kappa(\cdot,A) \equiv L_A$.
	Assumptions~\ref{item:kappa_bounded} and~\ref{item:kappa_square_integrability} of Theorem~\ref{T:Bianchi_etal} also hold since the underlying probability space $[N]^{(b)}$ is finite.
	Finally, assumption~\ref{item:f_C2} follows as well immediately from the definition of $\tilde{\calL}^{b}$, local smoothness of $\zeta$ and assumptions on $\tilde{\calL}_i$ for $i\in[N]$.
	We, therefore, apply Theorem~\ref{T:Bianchi_etal} with $\mathcal{K}=Q$, $T=1+\frac{N}{b}T^{\ast}$ and $\tilde{\varepsilon}=\Delta$.
	As a result, the DI problem associated with $f$,
	\begin{equation}\label{eq:DI_sgd_f}
	\dot{\chi}(t) \in 
	-\partial \EE f(\chi(t), A_b)
	\quad\text{for a.e. } t\in [0, T]
	\end{equation}
	is well-defined.
	Moreover, for $\delta>0$, there exists $\eta_0\le 1$ such that for a.e. $\eta\in (0,\eta_0)$ and for any family of $f$-SGD sequences $\mathcal{S} = \{\,(\chi_k^{\eta})_{k\in\NN}\colon \eta > 0\,\}$ defined above, 
	\begin{equation}
	\label{eq:PP}
	\PP\Big(
	\exists\;\chi\colon[0,T]\to\RR^D
	\text{ solving~\eqref{eq:DI_sgd_f} s.t. }
	\chi(0)\in Q
	\;\text{and}\;
	\sup_{t\in [0,T]} \vert \chi(t) - \bar{\chi}^{\eta}(t) \vert < \Delta
	\;\Big\vert\;
	\theta_0\in Q
	\Big)
	\ge 1-\delta,
	\end{equation}
	where $\bar{\chi}^{\eta}(t)$ is the piecewise interpolated process associated with some $(\chi_k^{\eta})_{k\in\NN}\in\mathcal{S}$ via~\eqref{eq:interpolated_def} (note that $\bar{\chi}^\eta$ is in fact a random variable measurable w.r.t. $\theta_0,\xi_1,\xi_2,\ldots$).

	Using the fact that $(\EE f(\cdot, A_b))_{\vert G_\Delta} = (\frac{b}{N}\tilde{\calL}(\cdot))_{\vert G_\Delta}$, we infer that if $\chi\colon[0,T]\to\RR^D$ solves~\eqref{eq:DI_sgd_f}, $\chi(0)\in Q$ and $\tau = \inf\{ t\ge 0\colon \chi(t)\notin G_\Delta\}$, then $\theta(t) \bydef \chi(\frac{N}{b}t)$ solves the DI
	\[
	\dot{\theta}(t) 
	\in  
	-\partial\tilde{\calL}(\theta(t))
	\quad \text{for a.e.} \quad t\in[0,b\tau/N].
	\]
	Since any $\theta$ initialized in $Q$ does not escape the set $G$ by assumption and since $G$ lies in the interior of $G_\Delta$, then $\tau=T$, i.e., $\chi$ remains in $G$.

	\begin{figure}[ht!]
		\includegraphics[width=0.3\textwidth]{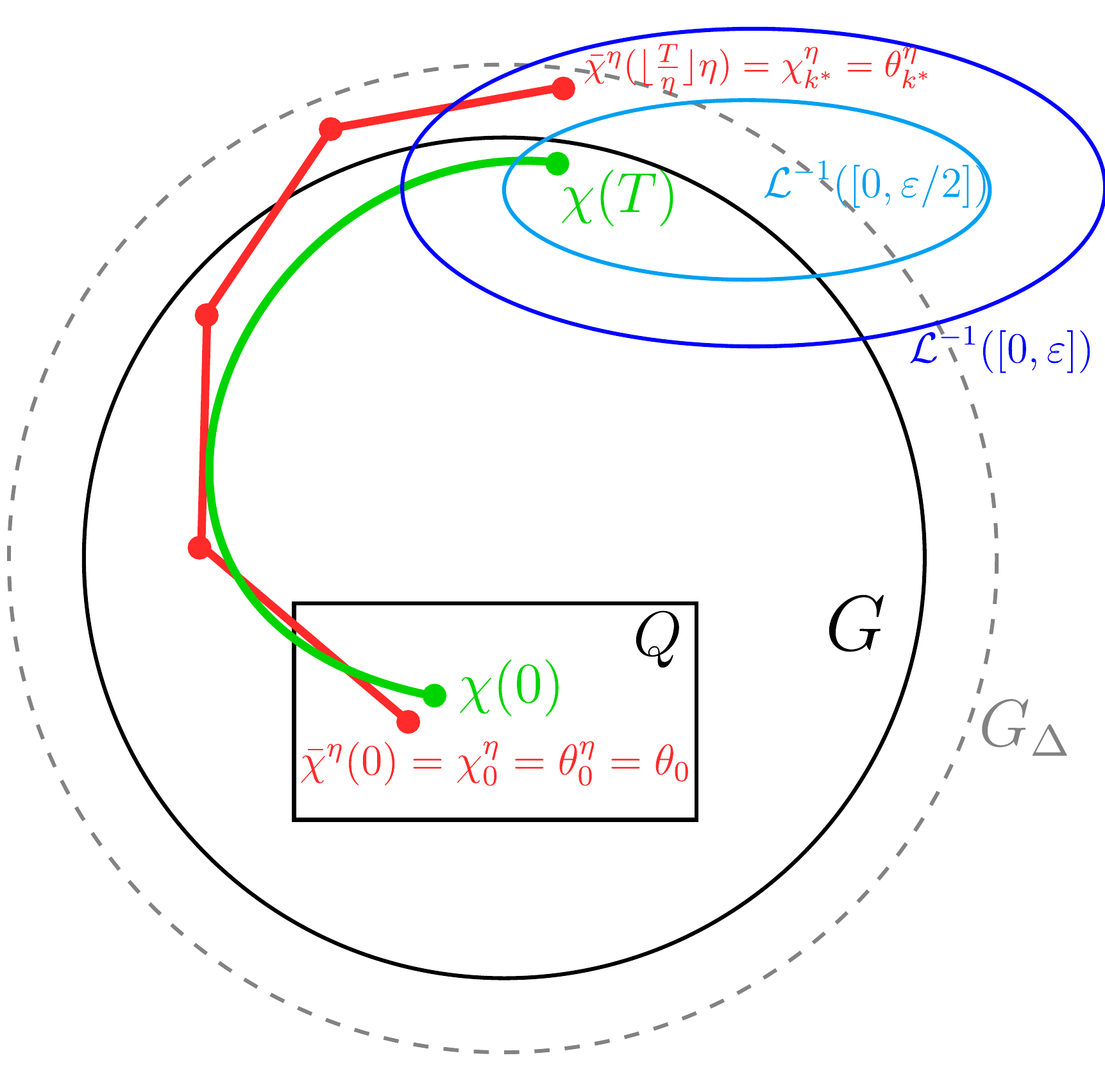}
		\centering
		\caption{Illustration of the sets and flows that appear in the proof of Theorem~\ref{P:abstract_convergence}}
	\end{figure}

	Let $\eta\in (0,\eta_0)$, $\theta_0\in Q$ be such that there exists $\chi\colon [0, T]\to \RR^D$ initialized in $Q$, which is $\Delta$ close to the (truncated) piecewise affine interpolated process $\bar{\chi}^\eta\colon [0, T]\to\RR^D$ associated with $(\chi_k^\eta)_{k\in \NN}$ defined above (with the initial condition $\chi_0^\eta=\theta_0$).
	By~\eqref{eq:PP}, this happens with probability at least $1-\delta$.
	By the discussion above, such $\chi$ never escapes the set $G$, whence $\bar{\chi}^\eta$ remains in the set $G_\Delta$.
	As a consequence, denoting $k^\ast = \lfloor \frac{T}{\eta} \rfloor$ to be maximal index such that $\chi_{k^\ast}^\eta$ lies on the curve $\bar{\chi}^\eta\colon [0,T]\to\RR^D$, we infer that $\chi_k^\eta$ also remains in $G_\Delta$ for all $k\in [k^\ast]$ and therefore $\tau^\eta > k^\ast$, which by definition implies that $\chi_k^\eta = \theta_k^{\eta}$ for all $k\in [k^\ast]$.

	Combining all the above observations, we deduce that with probability at least $1-\delta$, conditioned on a value of $\theta_0\in Q$, for a.e. $\eta\in(0,\eta_0)$ and $k^\ast= \lfloor \frac{T}{\eta} \rfloor$
	\begin{align*}
	\tilde{\calL}(\theta_{k^\ast}^\eta) 
	= 
	\tilde{\calL}\left( \chi_{k^\ast}^{\eta} \right)
	&=
	\tilde{\calL}\left( \bar{\chi}^\eta\left(
	\eta k^\ast 
	\right) \right)
	\\&\le 
	\tilde{\calL}\left( \chi\left(
	\eta k^\ast 
	\right) \right)
	+
	\sup\{\,
	\|v\|\colon v\in\partial f(x),\; x\in G_\Delta
	\,\}
	\cdot 
	\vert \chi( \eta k^\ast ) - \bar{\chi}(\eta k^\ast ) \vert 
	\\&= 
	\tilde{\calL}\Big( \theta\Big(
	\frac{\eta  b}{N}
	\Big\lfloor
	\frac{ NT^\ast}{ b\eta} + \frac{1}{\eta}
	\Big\rfloor 
	\Big) \Big) 
	+
	\sup\{\,
	\|v\|\colon v\in\partial f(x),\; x\in G_\Delta
	\,\}
	\cdot 
	\vert \chi( \eta k) - \bar{\chi}(\eta k) \vert 
	\\&\le 
	\frac{\varepsilon}{2}
	+
	L\Delta
	\le 
	\varepsilon
	\end{align*}
	as desired.
\end{proof}

\section{Training hidden layer only}\label{A:one_layer}
In this section, we apply the theory designed in the prequel to the setting in which only the first weight matrix $W$ is being updated using GD (whereas the second $V$ is fixed to some specific values). This is the training setup, which was introduced in \cite{oymak}. We combine the results from \cite{oymak} with our technique which results in a simpler approach as compared to the original one for showing overparametrization in such a training setup. However, the order of overparametrization that we obtain by proceeding is slightly off compared to the original work.

Consider the model with output dimension $d_2=1$, i.e.,
\[
    \RR^{N\times d_0} \ni X 
    \mapsto 
    \hat{y} 
    \bydef
    \phi(XW)v 
    \in \RR^N,
\]
where $X\in\RR^{d_0\times d_1}$, $W\in\RR^{d_0\times d_1}$, $v\in\RR^{d_1}$. 
In the sequel, we work under the following assumptions on the data and initialization, which are the same as in~\cite{oymak}.
\begin{assumption}\label{Asm:one_layer}
	Matrix $X$ has i.i.d. rows from the uniform distribution on the unit sphere in $\mathbb{R}^{d_0}$.
	The output vector $y$ is such that $\vert y_i \vert = \calO(1)$ for all $i$, so that $\|y\|_2=\calO(\sqrt{N})$.
	Vector $v$ is initialized s.t. half of its entries are $\|y\|_2/\sqrt{Nd_1}$ and the other half is set to $-\|y\|_2/\sqrt{Nd_1}$.
	Weight matrix $W$ has i.i.d. $\mathcal{N}(0,1)$ entries.
\end{assumption}
We restrict our attention to the case, where only $W$ is trained and $v$ remains fixed, i.e., we have $\calL=\calL(W)$ and we consider the following DI problem 
\begin{equation}\label{eq:DI_W}
	\frac{d}{ds} W(s)
	= 
	-\partial\calL(W)
	\quad
	\text{for a.e. }
	s\in \RR_+.
\end{equation}

The following theorem is the main result of this section.
\begin{theorem}\label{T:one_layer}
	Let the data satisfy Assumption~\ref{Asm:one_layer}.
	Assume that 
	\begin{equation}
		d_0 \in [\sqrt{N}, N],\quad 
		d_1 \ge 
		C\Bigl(
			\max\bigl(
				\frac{N\log N}{d_0},
				\frac{N^5}{d_0^4}
			\bigr)
		\Bigr)
	\end{equation}
	for some absolute constant $C\ge 0$.
	Then WHP any solution $W$ to the DI~\eqref{eq:DI_W} can be extended to a solution on $\RR_+$ and any such extension satisfies $\calL(W(t))\le \calL(W(0))\exp(-ct\beta_0^2)$ for some absolute $c>0$.
\end{theorem} 

Using Theorem~\ref{P:abstract_convergence}, we immediately obtain the following corollary to SGD iterates.

\begin{corollary}\label{cor:one_layer}
	Under the assumptions of Theorem~\ref{T:one_layer}, choose error $\varepsilon>0$, batch size $b=b(N)$ and any family $\{ (W_k^\eta)\colon \eta >0\}$ of $\calL^b$-SGD sequences~\eqref{eq:GD}.
	Then, WHP there exists a step size $\eta_0\in(0,1)$, s.t. for a.e. $\eta\in (0,\eta_0)$, $\calL(W^\eta_{k^\ast})<\varepsilon$ for sufficiently large $k^\ast$.
\end{corollary}

The proof of Theorem~\ref{T:one_layer} is presented at the end of this section and is an adaptation of the reasoning presented in Section~\ref{sec:di_globconv} together with some ideas from~\cite{oymak}, presented below.

In the case of training one layer only, one has 
\begin{align*}
    \partial \calL(W) 
    =
    X^T \dg(\hat{y}-y) \phi'(WX) \dg(v),
\end{align*}
where $\phi'(WX)_{ij}$ for $(i,j)\in[d_0]\times [d_1]$ is the Clark subdifferential of $W_{i:}X_{:j} \mapsto \phi(W_{i:}X_{:j})$ and $\phi'(WX) = \bigotimes_{ij} \phi'(WX)_{ij}$.
Let $R_s\in \phi'(XW(s))$ be such that $\dot{W}(s) = X^T\dg(\hat{y}(s)-y)R_s\dg(v)$ for a.e. $s$. 

For a vector $x=(x_1,\ldots,x_{d})\in\RR^{d}$ and integer $m\in\{1,\ldots ,d\}$, let $\|x\|_{m-}$ denote $m$-th smallest entry of $\vert x \vert = (\vert x_1 \vert,\ldots , \vert x_d \vert)$.
\begin{lemma}[Lemma C.2 in~\cite{oymak}]\label{lem:oymak_C2}
    If 
    \[ 
        \|W_s-W_0\|_F \le \sqrt{m} \min_{i\in[N]}\| X_{i:}W_0 \|_{m-},
    \] 
    then $\max_{i\in[N]}\| (R_s-R_0)_{i:} \|_2\le \sqrt{2m}$.
\end{lemma}

\begin{lemma}[Lemma C.3 in~\cite{oymak}]\label{lem:oymak_C3}
    If $\| X_{i:}\|_2 = 1$ for all $i$ and $W_0\in\RR^{d_0\times d_1}$ has i.i.d. $\mathcal{N}(0,1)$ entries, then 
    \[
        \min_{i\in [N]} \|X_{i:}W_0\|_{m-} \ge \frac{m}{2d_1}
        \quad
        \text{for all }
        i=1,2,\ldots,N
    \]
    with probability at least $1-Ne^{-m/6}$ for any $m\in\{1,2,\ldots,d_1\}$.
\end{lemma}
\begin{corollary}\label{cor:max_delta_R_estimate}
    If $\| X_{i:}\|_2 = 1$ for all $i$ and $W_0\in\RR^{d_0\times d_1}$ has i.i.d. $\mathcal{N}(0,1)$ entries, then 
    \[
        \max_{i\in[N]}\| (R_s-R_0)_{i:} \|_2
        \le
        10\sqrt{{\log N}} 
        + 
        2(d_1\|W_s-W_0\|_F)^{1/3}
    \]
    for all $s$ with probability at least $1-\frac{1}{N}$.
\end{corollary}
\begin{proof}
    Note that $\max_i\| (R_s-R_0)_{i:} \| \le \sqrt{d_1}$ always.
    Therefore, by Lemma~\ref{lem:oymak_C2} and Lemma~\ref{lem:oymak_C3}, if $\|W_s-W_0\|_F \le m\sqrt{m}/2d_1$, then $\max_{i\in[N]} \| (R_s-R_0)_{i:} \|_2\le \sqrt{2m}$ with probability at least $1-Ne^{-\min(m,d_1)/6}$ for any $m\in\NN$.
    Pick 
    \[
        m=
        \lceil
            12\log N + (2d_1\|W_s-W_0\|_F)^{2/3}
        \rceil  
    \]
    so that $\|W_s-W_0\|_F \le m\sqrt{m}/2d_1$.
    As $d_1\ge N$ (this is our global assumption) and 
    \[
        \lceil a + \log N \rceil 
        \le 
        a + \log N + 1
        \le 
        a + 4\log N
    \] 
    for $2\le N \in\NN$ and any $a>0$, then 
    \[
      12 \log N
      \le 
      m
      \le 
      48\log N + (2d_1\|W_s-W_0\|_F)^{2/3}.
    \]
    Therefore, using sub-additivity of the square root we get that 
    \[
        \max_i\| (R_s-R_0)_{i:}\|_2
        \le 
        \sqrt{2m}
        \le 
        10\sqrt{\log N} 
        + 
        2(d_1\|W_s-W_0\|_F)^{1/3}
    \]
    with probability at least $1-Ne^{-\min(m,d_1)/6}\ge 1-Ne^{-2\log N} = 1-\frac{1}{N}$, as desired.
\end{proof}
\begin{proposition}\label{prop:beta_0}
	Denote $\beta_0 \bydef \sigma_{min}( (X \star (R_0\dg(v)))^T )$.
    If 
	\[ 
		d_0\in[\sqrt{N},N] 
		\quad\text{and}\quad 
		d_1 \ge \max(N, C(N\log N)/d_0),
	\] 
	then 
	\[
		c\frac{\| y \|}{\sqrt{N}} 
		\le 
		\beta_0 
		\le 
		C
		\frac{\| y \|_2}{\sqrt{d_0}}	
	\] 
	WHP for some absolute constants $c,C>0$.
\end{proposition}
\begin{proof}
    To get the first estimate, apply \cite[Lemma~6.7]{oymak} and note that $\| v \|_2 = \sqrt{d_1}\| v \|_\infty$ and that for the uniform distribution on the unit sphere, $\|X\|_{op}\lesssim \sqrt{N/d_0}$ and $\lambda(X)\sim \operatorname{const}$ (note that $\beta_0 = \sigma_{min}(\mathcal{J}(W_0))$ in their notation).
    The second estimate follows from Weyl's inequality, cf.~\cite[Lemma~6.6]{oymak}. 
\end{proof}

\begin{proof}[Proof of Theorem~\ref{T:one_layer}]
	For any $W$, we have
	\begin{align*}
		\min_{w\in \partial L(W)} \| w \|_F^2
		&=
		\min_{R \in \phi'(XW)}
		\|  (X \star (R\dg(v)))^T (\hat{y}-y) \|_2^2
		\\ &\ge 
		2 \calL(W) \cdot 
		\min_{R \in \phi'(XW)} \sigma_{min}^2( (X \star (R\dg(v)))^T )
		,
	\end{align*}
	where in the first equality we have used Lemma~\ref{lem:mtx_inequality} below with $A=X^T$, $x=\hat{y}-y$ and $B=R\dg(v)$.
	Denote $\beta_s \bydef \sigma_{min}( (X \star (R_s\dg(v)))^T )$ and write $\calL(t)\bydef \calL(W(t))$ for short. 
	Then, by Proposition~\ref{P:DI_existence} and Gr\"{o}nwall's lemma 
	\begin{align}\label{eq:calL_decay_gronwall}
	\begin{split}
		\calL(t)
		\le 
		\calL(0)
		\exp\bigl(
			-2 \int_0^t \beta_s^2 \,ds
		\bigr)
		\le 
		\calL(0)
		\exp\bigl(
			-2t\beta_0^2
			+4\beta_0 \int_0^t \vert \beta_s - \beta_0 \vert \,ds
		\bigr).
	\end{split}
	\end{align}
	By Weyl's inequality, Lemma~\ref{lem:mtx_inequality} applied with $A=X^T$ and $B=(R_s-R_0)\dg(v)$ and by Corollary~\ref{cor:max_delta_R_estimate} we get
	\begin{align}\label{eq:beta_w_estimate}
	\begin{split}
		\vert
			\beta_s - \beta_0
		\vert
		&\le 
		\| (X \star ((R_s-R_0)\dg(v)))^T \|_{op}
		\\&=
		\sup_{x\colon \| x \|_2 = 1}
		\| (X \star ((R_s-R_0)\dg(v)) )^T x \|_2
		\\&\le 
		\| X \|_{op} \max_{i\in[N]} \| v_i(R_s-R_0)_{i:} \|_2
		\\&\le
		\| X \|_{op} \| v\|_\infty \max_{i\in[N]} \| (R_s-R_0)_{i:} \|_2
		\\&\le 
		\| X \|_{op} \| v\|_\infty
		\bigl[
			10\sqrt{{\log N}} 
			+ 
			2(d_1\|W_s-W_0\|_F)^{1/3}
		\bigr]
		,
	\end{split}
	\end{align}
	cf. Lemma C.1 in~\cite{oymak}.

	Passing with the norm under the integral, using Lemma~\ref{lem:mtx_inequality} and estimating $\| (R_s)_{i:} \|_2 \le \sqrt{d_1}$, we obtain
    \begin{align}\label{eq:DW_estimate}
    \begin{split}
        \|W_t-W_0\|_F
        &\le 
        \int_0^t
        \|\dot{W}_s\|_F \,ds
        \\&=
        \int_0^t
        \|
        (X\ast (R_s\dg(v)))^T (\hat{y}(s)-y))
        \|_2 
        \,ds
        \\&\le 
        \sqrt{2}
        \int_0^t
        \|
        X\ast (R_s\dg(v))
        \|_{op}
        \sqrt{\calL(s)}
        \,ds
		\\&\le 
        \sqrt{2}
        \| X \|_{op}
        \| v \|_\infty 
		\int_0^t
		\sqrt{\calL(s)}
		\max_{i\in[N]}
		\| (R_s)_{i:} \|_2
		\,ds  
		\le 
        \sqrt{2d_1}
        \| X \|_{op}
        \| v \|_\infty 
		\bar{\calL}(t)
    \end{split}
    \end{align}
	Combining~\eqref{eq:calL_decay_gronwall},~\eqref{eq:beta_w_estimate} and~\eqref{eq:DW_estimate}, noting that $\|X\|_{op}\|v\|_{\infty} \le C\frac{\| y\|_2}{\sqrt{d_0d_1}}$ WHP for some absolute constant $C>0$ and using monotonicity of $\mathcal{L}$, we get that 
	\begin{align}
	\begin{split}
		\calL(t)
		&\le 
		\calL(0)
		\exp\bigl(
			-2t\beta_0^2
			+4\beta_0
			\| X \|_{op} \| v\|_\infty
			\int_0^t \bigl[
				10\sqrt{{\log N}} 
				+ 
				2(d_1\|W_s-W_0\|_F)^{1/3}
			\bigr] \,ds
		\bigr)
		\\&\le 
		\calL(0)
		\exp\Bigl(
			-2t\beta_0^2
			+C\frac{\beta_0 \| y\|_2}{\sqrt{d_0d_1}}
			\int_0^t \bigl[
				\sqrt{{\log N}} 
				+ 
				(d_1\|W_s-W_0\|_F)^{1/3}
			\bigr] \,ds
		\Bigr)
		\\&\le 
		\calL(0)
		\exp\Bigl(
			-2t\beta_0^2
			+C\frac{\beta_0 \| y\|_2}{\sqrt{d_0d_1}}
			\int_0^t \bigl[
				\sqrt{{\log N}} 
				+ 
				\bigl(
					\frac{d_1\| y\|_2 \bar{\calL}(s)}{\sqrt{d_0}}
				\bigr)^{1/3}
			\bigr] \,ds
		\Bigr)
		\\&\le 
		\calL(0)
		\exp\Bigl(
			-2t\beta_0^2
			+Ct\frac{\beta_0 \| y\|_2}{\sqrt{d_0d_1}}
			\bigl[
				\sqrt{{\log N}} 
				+ 
				\bigl(
					\frac{d_1\| y\|_2 \bar{\calL}(t)}{\sqrt{d_0}}
				\bigr)^{1/3}
			\bigr]
		\Bigr)
		\\&\le 
		\calL(0)
		\exp\Bigl(
			-2t\beta_0^2
			\Bigl[
			1 -
			C
			\frac{\| y\|_2 }{\beta_0\sqrt{d_0d_1}}
			\bigl[
				\sqrt{{\log N}} 
				+ 
				\bigl(
					\frac{d_1\| y\|_2 \bar{\calL}(t)}{\sqrt{d_0}}
				\bigr)^{1/3}
			\bigr]
			\Bigr]
		\Bigr).
	\end{split}
	\end{align}
	Therefore, $\bar{\calL}$ satisfies the following differential inequality 
	\[
		y(0)=0,\quad 
		y'(t)
		\le 
		\sqrt{\calL}(0)
		\exp\Bigl(
			-2t\beta_0^2
			\Bigl[
			1 -
			C
			\frac{\| y\|_2 }{\beta_0\sqrt{d_0d_1}}
			\bigl[
				\sqrt{{\log N}} 
				+ 
				\bigl(
					\frac{d_1\| y\|_2 y(t)}{\sqrt{d_0}}
				\bigr)^{1/3}
			\bigr]
			\Bigr]
		\Bigr).
	\]
	Reasoning in the same way as in Lemma~\ref{L:ODE_solutions}, we get that $\calL(t)\le \calL(0)\exp(-ct\beta_0^2)$ if the following condition holds asymptotically at initialization:
	\begin{equation}\label{eq:cond_1_one_layer}
		\frac{\| y\|_2 }{\beta_0\sqrt{d_0d_1}}
		\bigl[
			\sqrt{{\log N}} 
			+ 
			\bigl(
				\frac{d_1\| y\|_2 }{\sqrt{d_0}}
				\cdot 
				\frac{\sqrt{\calL(0)}}{\beta_0^2}
			\bigr)^{1/3}
		\bigr]
		= 
		o(1).
	\end{equation}
	Using Proposition~\ref{prop:beta_0} we get that $\|y\|_2/\beta_0 \lesssim \sqrt{N}$.
	Since $\sqrt{\calL}(0) \lesssim \| y \|_2$ WHP, then~\eqref{eq:cond_1_one_layer} is implied by 
	\begin{equation*}
		\frac{\sqrt{N}}{\sqrt{d_0d_1}}
		\bigl[
			\sqrt{{\log N}} 
			+ 
			\bigl(
				d_1 \frac{N}{\sqrt{d_0}}
			\bigr)^{1/3}
		\bigr]
		= 
		o(1),
	\end{equation*}
	which is equivalent to 
	\begin{equation*}
		d_1 
		\ge 
		\Omega\Bigl(
			\max\bigl(
				\frac{N\log N}{d_0},
				\frac{N^5}{d_0^4}
			\bigr)
		\Bigr)
	\end{equation*}
	as desired.
\end{proof}

\section{Linear algebra lemmas}
\begin{lemma}\label{L:mtx_product_norm_inequality}
	If $A_1,\ldots,A_k$ are any matrices such that $A_1\cdots A_k$ is well defined, then\begin{align}\label{eq:mtx_product_norm_inequality}
	\| A_1\cdots A_k \|_F &\le 
	\min_{j\in [k]}
	\| A_j \|_{F} \prod_{i\neq j} \|A_{i}\|_{op}.
	\end{align}
\end{lemma}
\begin{proof}
	Indeed, for any $j\in[k]$, let $A^{\leftarrow j}=A_1\cdots A_{j-1}$ and $A^{j\rightarrow}=A_{j+1}\cdots A_k$, where we identify an empty product with the identity operator.
	Then, for any $j\in [k]$
	\begin{align*}
	\| A_1\cdots A_k \|_F^2
	&= 
	\sum_{i}
	\| A^{\leftarrow j} (A_j \cdot  A^{j\rightarrow})_{:i} \|^2
	\\&\le 
	\| A^{\leftarrow j} \|_{op}^2
	\sum_i
	\| (A_j \cdot  A^{j\rightarrow})_{:i} \|^2
	\\&= 
	\| A^{\leftarrow j} \|_{op}^2
	\| A_j A^{j\rightarrow} \|^2_F
	\\&=
	\| A^{\leftarrow j} \|_{op}^2
	\sum_{i}
	\| (A_j)_{i:} \cdot A^{j\rightarrow}\|^2
	\le 
	\| A^{\leftarrow j} \|_{op}^2
	\| A^{j\rightarrow} \|_{op}^2
	\| A_j \|_F^2	
	\end{align*}
	and~\eqref{eq:mtx_product_norm_inequality} follows by taking square roots, using sub-multiplicity of the operator norm and taking minimum over all possible choices of $j\in [k]$.
\end{proof}

\begin{lemma}\label{lem:mtx_inequality}
    For any matrices $A,B$ and any vector $x$ such that $A\dg(x)B$ exists one has
    \begin{equation}\label{eq:norm_estimate_1}
        \| (A^T \star B)^T x \|_2
        =
        \| A \dg(x) B \|_F 
        \le 
        \| A \|_{op}
        \| x \|_2 
        \max_i \| B_{i:} \|_2.
    \end{equation}
\end{lemma}
\begin{proof}
    Square and expand both hand sides to check that $\| (A^T \star B)^T x \|_2 = \| A \dg(x) B \|_F $. 
    The inequality follows as 
    \begin{align*}
        \| A \dg(x) B \|_F^2 
        &\le
        \| A \|_{op}^2
        \| \dg(x) B \|_F^2 
        \\&= 
        \| A \|_{op}^2
        \sum_i x_i \| B_{i:} \|_2^2
        \le 
        \| A \|_{op}^2
        \| x \|_2^2
        \max_i \| B_{i:} \|_2^2
    \end{align*}
    as desired.
\end{proof}

\end{document}